%% file: main.tex
\documentclass[conference]{IEEEtran}
\usepackage{times}

\usepackage[numbers]{natbib}
\usepackage{multicol}
\usepackage[bookmarks=true]{hyperref}
\usepackage{graphics} 
\usepackage{amsmath} 
\usepackage{amsthm}
\usepackage{verbatim}
\usepackage{graphicx}
\usepackage{caption} 
\usepackage{subcaption}

\usepackage{graphics} 
\usepackage{multirow}
\usepackage{url}

\newtheorem{theorem}{Theorem}
\newtheorem{lemma}[theorem]{Lemma}
\newtheorem{corollary}{Corollary}[theorem]

\title{\LARGE \bf
I-Planner: Intention-Aware Motion Planning\\
Using Learning Based Human Motion Prediction
}

\author{Jae Sung Park and Chonhyon Park and Dinesh Manocha\\
Department of Computer Science, UNC Chapel Hill, NC, USA \\
\{jaesungp, chpark, dm\}@cs.unc.edu \\
\url{http://gamma.cs.unc.edu/SafeMP} (video included)
\vspace*{-.1in}
}

\begin{document}

\maketitle
\thispagestyle{empty}
\pagestyle{empty}


\begin{abstract}
We present a motion planning algorithm to compute collision-free and smooth trajectories for high-DOF robots interacting with humans in a shared workspace.
Our approach uses offline learning of human actions along with temporal coherence to predict the human actions. Our intention-aware online planning algorithm uses the learned database to compute a reliable trajectory based on the predicted actions. We represent the predicted human motion using a Gaussian distribution and compute tight upper bounds on collision probabilities for safe motion planning.
We also describe novel techniques to account for noise in human motion prediction.
We highlight the performance of our planning algorithm in complex simulated scenarios and real world benchmarks with 7-DOF robot arms operating in a workspace with a human performing complex tasks.   
We demonstrate the benefits of our intention-aware planner in terms of computing safe trajectories in such uncertain environments.
\end{abstract}

\input{1.tex}
\input{2.tex}
\input{3.tex}

\input{4.tex}
\input{5.tex}
\input{6.tex}
\input{7.tex}
\input{8.tex}
\input{9.tex}

\bibliographystyle{plainnat}
\bibliography{refs}

\end{document}

%% file: 1.tex
\section{Introduction}

Motion planning algorithms are used to compute collision-free paths for robots among obstacles.
As robots are increasingly used in workspace with moving or unknown obstacles, it is important to develop reliable planning algorithms that can handle environmental uncertainty and the dynamic motions. In particular, we address the problem of planning safe and reliable motions for a robot that is working in environments with humans. As the humans move, it is important for the robots to predict the human actions and motions from sensor data and to compute appropriate trajectories.

\begin{figure}[ht]
  \centering
  \begin{subfigure}[t]{0.32\linewidth}
    \centering
    \includegraphics[width=\textwidth]{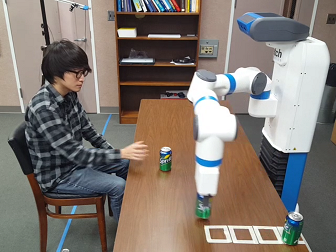}
  \end{subfigure}
  \begin{subfigure}[t]{0.32\linewidth}
    \centering
    \includegraphics[width=\textwidth]{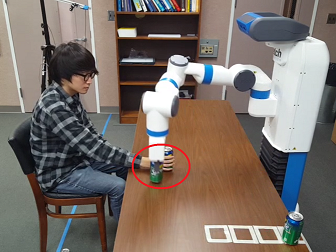}
    \caption{}
  \end{subfigure}
  \begin{subfigure}[t]{0.32\linewidth}
    \centering
    \includegraphics[width=\textwidth]{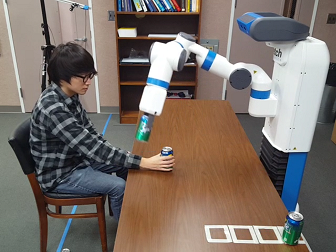}
  \end{subfigure}
  \\
  \begin{subfigure}[t]{0.32\linewidth}
    \centering
    \includegraphics[width=\textwidth]{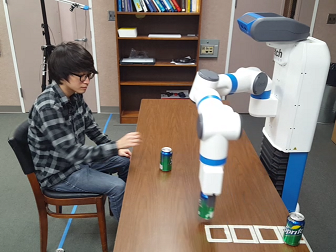}
  \end{subfigure}
  \begin{subfigure}[t]{0.32\linewidth}
    \centering
    \includegraphics[width=\textwidth]{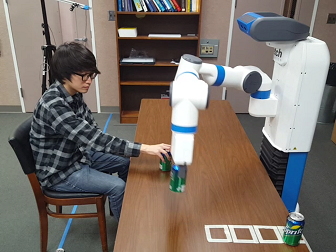}
    \caption{}
  \end{subfigure}
  \begin{subfigure}[t]{0.32\linewidth}
    \centering
    \includegraphics[width=\textwidth]{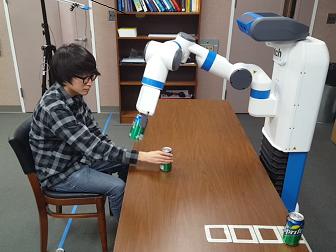}
  \end{subfigure}
  \caption{A 7-DOF Fetch robot is moving its arm near a human, avoiding collisions.
           (a) While the robot is moving, the human tries to move his arm to block the robot's path. The robot arm trajectory is planned without human motion prediction, which may result in collisions and a jerky trajectory, as shown with the red circle. This is because the robot cannot respond to the human motion to avoid collisions. 
           (b) The trajectory is computed using our human motion prediction algorithm; it avoids collisions and results in smoother trajectories. The robot trajectory computation uses collision probabilities to anticipate the motion and compute safe trajectories.  
           }
  \label{fig:result_real_robot}
\end{figure}
In order to compute reliable motion trajectories in such shared environments, it is important to gather the state of the humans as well as predict their motions. There is considerable work on online tracking and prediction of human motion in computer vision and related areas~\cite{shotton2013real}. However, the current state of the art in gathering motion data results in many challenges. First of all, there are errors in the data due to the sensors (e.g., point cloud sensors) or poor sampling~\cite{choo2014statistical}. Secondly, human motion can be sudden or abrupt and this can result in various uncertainties in terms of accurate representation of the environment. One way to overcome some of these problems is to use predictive or estimation techniques for human motion or actions, such as using various filters like Kalman filters or particle filters~\cite{vasquez2009growing}. Most of these prediction algorithms use a motion model that can predict future motion based on the prior positions of human body parts or joints, and corrects the error between the estimates and actual measurements.
In practice, these approaches only work well when there is sufficient information about prior motion that can be accurately modeled by the underlying motion model. 
In some scenarios, it is possible to infer high-level human intent using additional information, and thereby perform a better prediction of future human motions~\cite{bandyopadhyay2013intention,beraglmp}. These techniques are used to predict the pedestrian trajectories based on environmental information in 2D domains. 

\noindent {\bf Main Results:} We present a novel high-DOF motion planning approach to compute collision-free trajectories for robots operating in a workspace with human-obstacles or human-robot cooperating scenarios (I-Planner). Our approach is general, and doesn't make much assumptions about the environment or the human actions. We track the positions of the human using depth cameras and present a new method for human action prediction using combination of classification and regression methods. Given the sensor noises and prediction errors, our online motion planner uses probabilistic collision checking to compute a high dimensional robot trajectory that tends to compute safe motion in the presence of uncertain human motion.
As compared to prior methods, the main benefits of our approach include:
\begin{enumerate}
\item A novel data-driven algorithm for intention and motion prediction, given noisy point cloud data. Compared to prior methods, our formulation can account for big noise in skeleton tracking  in terms of human motion prediction.
\item An online high-DOF robot motion planner for efficient completion of collaborative human-robot tasks that uses upper bounds on collision probabilities to compute safe trajectories in challenging 3D workspaces. Furthermore, our trajectory optimization based on probabilistic collision checking results in smoother paths. 
\end{enumerate}
We highlight the performance of our algorithms in a simulator with a 7-DOF KUKA arm operating and in a real world setting with a 7-DOF Fetch robot arm in a workspace with a moving human performing cooperative tasks. We have evaluated its performance in some challenging or cluttered 3D environments where the human is  close to the robot and moving at varying speeds.
We demonstrate the benefits of our intention-aware planner in terms of computing safe trajectories in these scenarios. 
A preliminary version of this paper was published~\cite{park2017intention}.
As compared to~\cite{park2017intention}, we improve the human motion prediction algorithm using depth sensor data. We present a mathematical analysis on the robustness of our prediction algorithm and highlight its benefits on challenging scenarios in terms of improved accuracy. We also analyze the performance of our algorithm with varying human motion speeds.

The rest of paper is organized as follows. 
In Section~\ref{sec:related}, we give a brief survey of prior work. 
Section~\ref{sec:overview} presents an overview of our human intention-aware motion planning algorithm. Offline learning and runtime prediction of human motions are described in Section~\ref{sec:prediction}, and these are combined with our optimization based motion planning algorithm in Section~\ref{sec:planning}.
The performance of the intention-aware motion planner is analyzed in Section~\ref{sec:analysis}.
Finally, we demonstrate the performance of our planning framework for a 7-DOF robot in Section~\ref{sec:results}.

%% file: 2.tex
\section{Related work}
\label{sec:related}

In this section, we give a brief overview of  prior work on human motion prediction, task planning for human-robot collaborations, and motion planning in environments shared with humans.

\subsection{Intention-aware motion planning and prediction}
Intention-Aware Motion Planning (IAMP) denotes a motion planning framework where the uncertainty of human intention is taken into account~\cite{bandyopadhyay2013intention}.
The goal position and the trajectory of moving pedestrians can be considered as human intention and used so that a moving robot can avoid pedestrians~\cite{unhelkar2015human}.

In terms of robot navigation among obstacles and pedestrians, accurate predictions of humans or other robot positions are possible based on crowd motion models~\cite{fulgenzi2007dynamic,van2008interactive} or integration of motion models with online learning techniques~\cite{kim2014brvo} for 2D scenarios and they are orthogonal to our approach.

Predicting the human actions or the high-DOF human motions has several challenges.
Estimated human poses from recorded videos or realtime sensor data tend to be inaccurate or imperfect due to occlusions or limited sensor ranges~\cite{choo2014statistical}.
Furthermore, the whole-body motions and their complex dynamics with many high-DOF makes it difficult to represent them with accurate motion models~\cite{hofmann2012multi}.
There has been a considerable literature on recognizing human actions~\cite{turaga2008machine}.
Machine learning-based algorithms using Gaussian Process Latent Variable Models (GP-LVM)~\cite{ek2007gaussian,urtasun20063d} or Recurrent neural network (RNN)~\cite{fragkiadaki2015recurrent} have been proposed to compute accurate human dynamics models.
Recent approaches use the learning of human intentions along with additional information, such as temporal relations between the actions~\cite{nikolaidis2013human,hawkins2013probabilistic} or object affordances~\cite{koppula2016anticipating} to improve the accuracy.
Inverse Reinforcement Learning (IRL) has been used to predict 2D motions~\cite{ziebart2008maximum,kuderer2012feature} or 3D human motions~\cite{dragan2013policy}.

\subsection{Robot task planning for human-robot collaboration}
In human-robot collaborative scenarios, robot task planning algorithms have been developed for the efficient distribution of subtasks.
One of their main goal is reducing the completion time of the overall task by interleaving subtasks of robot with subtasks of humans with well designed task plans.
In order to compute the best action policy for a robot, Markov Decision Processes (MDP) have been widely used~\cite{busoniu2008comprehensive}.
Nikolaidis et al.~\cite{nikolaidis2013human} use MDP models based on mental model convergence of human and robots.
Koppula and Saxena~\cite{koppula2016anticipatory} use Q-learning to train MDP models where the graph model has the transitions corresponding to the human action and robot action pairs.
P{\'e}rez-D'Arpino and Shah~\cite{perez2015fast} used a Bayesian learning algorithm on hand motion prediction and tested the algorithm in a human-robot collaboration tasks.
Our MDP models extend these approaches, but also take into account the issue of avoiding collisions between the human and the robot.

\subsection{Motion planning in environments shared with humans}

Prior work on motion planning in the context of human-robot interaction has focused on computing robot motions that satisfy cognitive constraints such as social acceptability~\cite{sisbot2007human} or being legible to humans~\cite{dragan2015effects}.

In human-robot collaboration scenarios where both the human and the robot perform manipulation tasks in a shared environment, it is important to compute robot motions that avoid collisions with the humans for safety reasons. 
Dynamic window approach~\cite{fox1997dynamic} (which searches the optimal velocity in a short time interval) and online replanning~\cite{SMP:2005,Park:2012:ICAPS,GPUITOMP} (which interleaves planning with execution) are widely used approaches for planning in such dynamic environments.
As there are uncertainties in the prediction model and in the sensors for human motion, the future pose is typically represented as the \emph{Belief} state, which corresponds to the probability distribution over all possible human states.
Mainprice and Berenson~\cite{mainprice2013human} explicitly construct an occupied workspace voxel map from the predicted Belief states of humans in the shared environment and avoid collisions.

Partially Observable Markov Decision Process (POMDP) techniques are widely used in motion planning with uncertainty in the robot state and the environment.
These approaches first estimate the robot state and the environment state, represent them in a probabilistic manner, and tend to compute the best action or the best robot trajectory considering likely possibilities.
Because the search space of exact POMDP formulation is too large, many practical and approximate POMDP solutions have been proposed~\cite{kurniawati2011motion,van2012motion} to reduce the running time and obtain almost realtime performance.
\cite{bai2015intention} use an approximate and realtime POMDP motion planner on autonomous driving carts. Our algorithm solves the problem in two steps: first the future human motion trajectory is predicted; next our planning algorithm generates a collision-free trajectory by taking into account the predicted trajectory. Our current formulation does not fully account for the uncertainty in the robot state and can be combined with POMDP approaches to handle this issue. 

%% file: 3.tex
\section{Overview}
\label{sec:overview}

In this section, we first introduce the notation and terminology used in the paper and give an overview of our motion planning algorithms.

\subsection{Notation and assumptions}
\label{subsec:notation_learning}

As we need to learn about human actions and short-term motions, a large training dataset of human motions is needed. 
We collect $N$ demonstrations of how human joints typically move while performing some tasks and in which order subtasks are performed.
Each demonstration is represented using $T^{(i)}$ time frames of human joint motion, where the superscript $(i)$ represents the demonstration index. 
The motion training dataset is represented as following:
\begin{itemize}
\item $\xi$ is a matrix of tracked human joint motions. 
      $\xi^{(i)}$ has $T^{(i)}$ columns, where a column vector represents the different human joint positions during each time frame.
\item $F$ is a feature vector matrix.
      $F^{(i)}$ has $T^{(i)}$ columns and is computed from $\xi^{(i)}$. 
\item $\mathbf{a}^h$ is a human action (or subtask) sequence vector that represents the action labels over different time frames. 
      For each time frame, the action is categorized into one of the $m^h$ discrete action labels, where the action label set is $A^h = \lbrace a_1^h, \cdots, a_{m^h}^h \rbrace$.
\item $ L = \lbrace (\xi^{(i)}, F^{(i)}, \mathbf{a}^{h,(i)} ) \rbrace _{i=1}^N $ is the motion database used for training.
      It consists of human joint motions, feature descriptors and the action labels at each time frame.
\end{itemize}

During runtime, MDP-based human action inference is used in our task planner.
The MDP model is defined as a tuple $(P, A^r, T)$:
\begin{itemize}
\item $A^r = \lbrace a_1^r, \cdots, a_{m^r}^r\rbrace$ is a robot action (or subtask) set of $m^r$ discrete action labels.
\item $P = \mathcal{P}(A^r \cup A^h)$, a power set of union of $A^r$ and $A^h$, is the set of states in MDP.
      We refer to the state $\mathbf{p} \in P$ as a \emph{progress state}  because each state represents which human and robot actions have been performed so far.
      We assume that the sequence of future actions for completing the entire task depends on the list of actions completed.
      $\mathbf{p}$ has $m^h + m^r$ binary elements, which represent corresponding human or robot actions have been completed ($\mathbf{p}_j = 1$) or not ($\mathbf{p}_j = 0$).
      For cases where same actions can be done more than once, the binary values can be replaced with integers, to count the number of actions performed.
\item $T : P \times A^r \rightarrow \Pi(P)$ is a transition function.
      When a robot performs an action $a^r$ in a state $\mathbf{p}$, $T(\mathbf{p}, a^r)$ is a probability distribution over the progress state set $P$.
      The probability of being state $\mathbf{p}'$ after taking an action $a^r$ from state $\mathbf{p}$ is denoted as $T(\mathbf{p}, a^r, \mathbf{p}')$.
\item $\pi : P \rightarrow A^r$ is the action policy of a robot.
      $\pi(\mathbf{p})$ denotes the best robot action that can be taken at state $\mathbf{p}$, which results in maximal performance.
\end{itemize}
We use the Q-learning~\cite{sutton1998reinforcement} to determine the best action policy during a given state, which rewards the values that are induced from the result of the execution. 

\subsection{Robot representation}
\label{subsec:notation_robot}

\begin{figure}[t]
  \centering
  \includegraphics[trim=0in 0in 0in 0in, clip=true, width=0.9\linewidth]{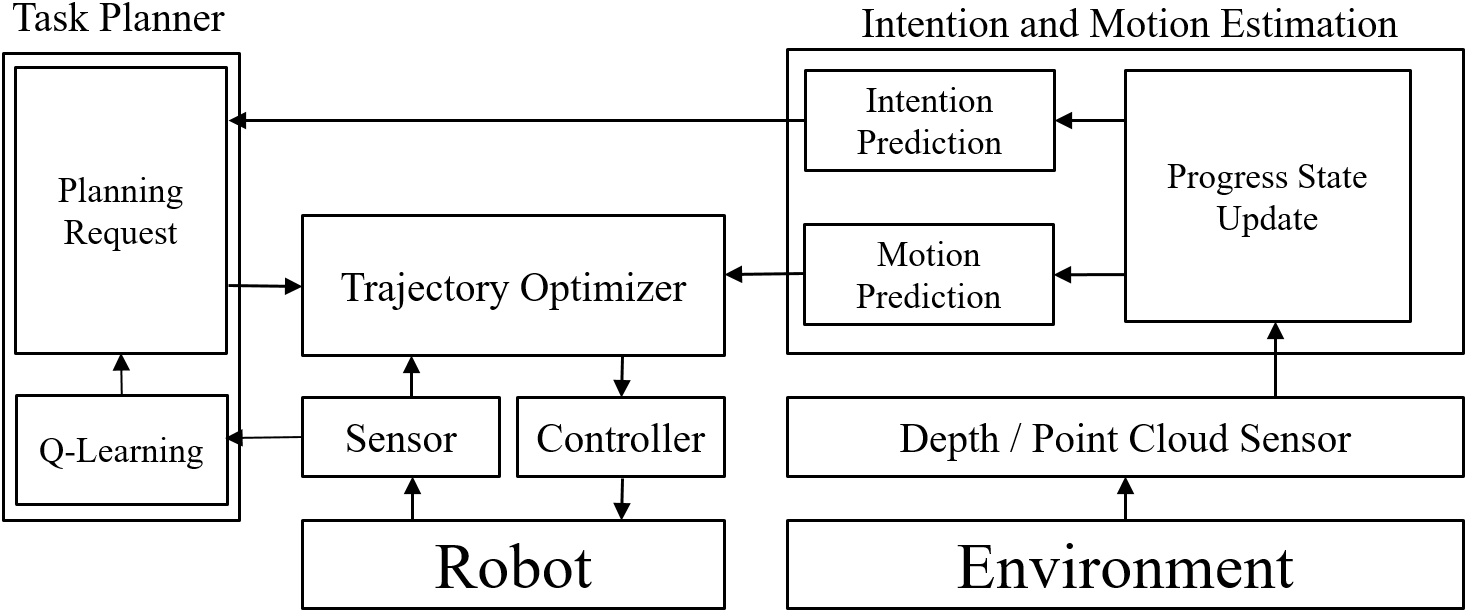}
  \caption{{\bf Overview of our Intention-Aware Planner:} {Our approach consists of three main components: task planner, trajectory optimization, and intention and motion estimation. }}
  \label{fig:architecture}
\end{figure}

We denote a single configuration of the robot as a vector $\mathbf{q}$ that consists of joint-angles. 
The $n$-dimensional space of configuration $\mathbf{q}$ is the configuration space $\mathcal{C}$.
We represent each link of the robot as $R_i$.
The finite set of bounding spheres for link $R_i$ is $\lbrace B_{i1}, B_{i2}, \cdots \rbrace$, and is used as a bounding volume of the link, i.e., $R_i \subset \cup_j B_{ij}$.
The links and bounding spheres at a configuration $\mathbf{q}$ are denoted as $R_i(\mathbf{q})$ and $B_{ij}(\mathbf{q})$, respectively.
In our benchmarks, where the robot arms are used, these bounding spheres are automatically generated using the medial axis of robot links. 
We also generate the bounding spheres $\lbrace C_1, C_2, \cdots \rbrace$ for humans and other obstacles.

For a planning task with start and goal configurations $\mathbf{q}_s$ and $\mathbf{q}_g$, the robot's trajectory is represented by a matrix $\mathbf{Q}$,
\begin{equation*}
\mathbf{Q} = \begin{bmatrix}
    \mathbf{q}_s & \mathbf{q}_1 & \cdots & \mathbf{q}_{n-1} & \mathbf{q}_g \\
    t_0 & t_1 & \cdots & t_{n-1} & t_n
\end{bmatrix},
\end{equation*}
where robot trajectory passes through the $n+1$ waypoints.
We denote the $i$-th waypoint of $\mathbf{Q}$ as $\mathbf{x}_i = \begin{bmatrix} \mathbf{q}_i^T & t_i \end{bmatrix}$.

\subsection{Online motion planning}

The main goals of our motion planner are: (1) planning high-level tasks for a robot by anticipating the most likely next human action and (2) computing a robot trajectory that reduces the probability of collision between the robot and the human or other obstacles, by using motion prediction.

At the high-level task planning step, we use MDP, which is used to compute the best action policies for each state.
A state of an MDP graph denotes the progress of the whole task.
The best action policies are determined through reinforcement learning with Q-learning.
Then, the best action policies are updated within the same state. The probability of choosing the action increases or decreases according to the reward function. 
Our reward computation function is affected by the prediction of intention and the delay caused by predictive collision avoidance.

We also estimate the short-term future motion from learned information in order to avoid future collisions.
From the joint position information, motion features are extracted based on human poses and surrounding objects related to human-robot interaction tasks, such as joint positions of humans, relative positions from a hand to other objects, etc.
The motions are classified over the human action set $A^h$.
For classifying the motions, we use Import Vector Machine (IVM)~\cite{zhu2012kernel} for classification and a Dynamic Time Warping (DTW)~\cite{muller2007dynamic} kernel function for incorporating the temporal information.
Given the human motions database of each action type, we train future motions using Sparse Pseudo-input Gaussian Process (SPGP)~\cite{snelson2005sparse}.
Combining these two prediction results, the final future motion is computed as the weighted sum over different action types weighed by the probability of each action type that could be performed next.
For example, if the action classification results in probability 0.9 for action \emph{Move forward} and 0.1 for action \emph{Move backward}, the future motion prediction (the results of SPGP) for \emph{Move forward} dominates. If the action classification results in probability 0.5 for both actions, the predicted future motions for each action class will be used in avoiding collisions but with weights $0.5$. However, in this case, the current motion does not have specific features to distinguish the action, meaning that the future motion in a short term will be similar and there will be an overlapped region in 3D space, working as a future motion of weight 1. More details are described in Section~\ref{sec:prediction}.

After deciding which robot task will be performed, the robot motion trajectory is then computed that tends to avoid collisions with humans.
An optimization-based motion planner~\cite{Park:2012:ICAPS} is used to compute a locally optimal solution that minimizes the objective function subject to many types of constraints such as robot related constraints (e.g., kinematic constraint), human motion related constraints (e.g., collision free constraint), etc.
Because future human motion is uncertain, we can only estimate the probability distribution of the possible future motions.
Therefore, we perform probabilistic collision checking to reduce the collision probability in future motions.
We also continuously track the human pose and update the predicted future motion to re-plan safe robot motions.
Our approach uses the notion of online probabilistic collision detection~\cite{pan2011probabilistic,parkfast,park2016efficient} between the robot and the point-cloud data corresponding to human obstacles, to compute reactive costs and integrate them with our optimization-based planner.

%% file: 4.tex
\section{Human action prediction}
\label{sec:prediction}

In this section, we describe our human action prediction algorithm, which consists of offline learning and online inference of actions.

\subsection{Learning of human actions and temporal coherence}
\label{subsec:offline}

\begin{figure}[t]
  \centering
  \begin{subfigure}[]{0.47\linewidth}
    \centering
    \includegraphics[height=1.3in]{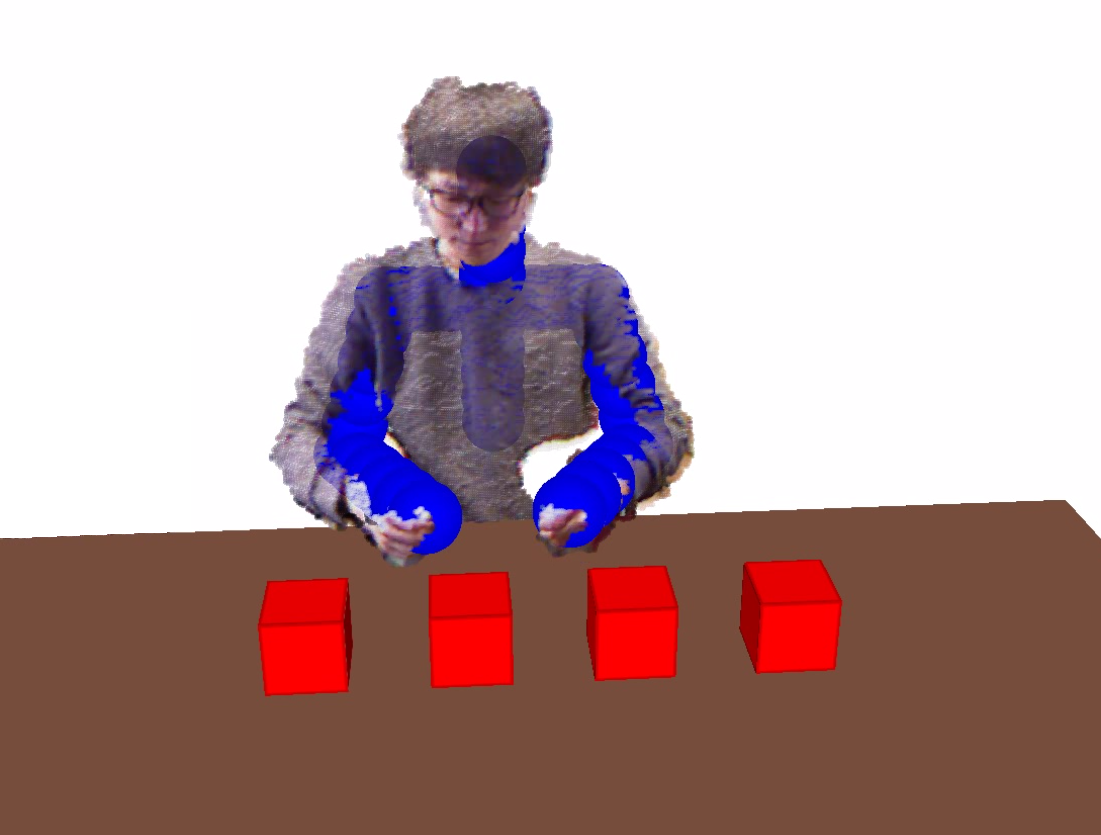}
    \caption{}
  \end{subfigure}
  \begin{subfigure}[]{0.47\linewidth}
    \centering
    \includegraphics[height=1.3in]{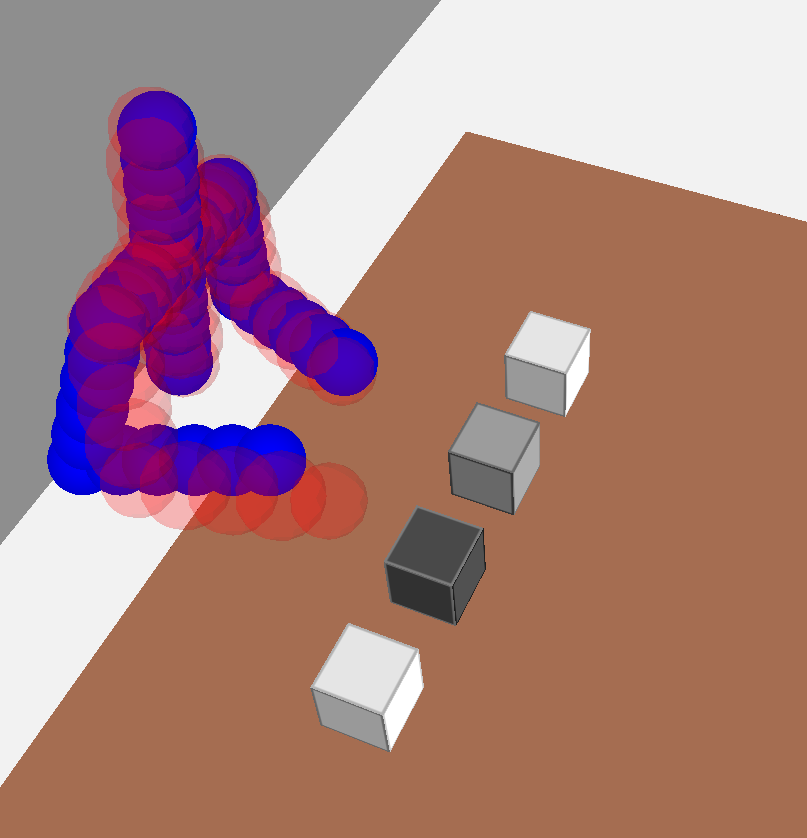}
    \caption{}
  \end{subfigure}
  \caption{{\bf Motion uncertainty and prediction:} 
           {(a) A point cloud and the tracked human (blue spheres). The joint positions are used as feature vectors.
           (b) Prediction of next human action and future human motion, where 4 locations are colored according to their probability of next human action from white (0\%) to black (100\%).
           Prediction of future motion after 1 second (red spheres) from current motion (blue spheres) is shown as performing the action: \textit{move right hand to the second position} which has the highest probability associated with it.}
           }
\end{figure}

We collect $N$ demonstrations to form a motion database $L$.
The 3D joint positions are tracked using OpenNI library~\cite{PrimeSense2010}, and their coordinates are concatenated to form a column vector $\xi^{(i)}$.
For full-body motion prediction, we used 21 joints, each of which has 3D coordinates tracked by OpenNI. So, $\xi^{(i)}$ is a 63-dimensional vector.
For upper-body motion prediction, there are 10 joints and thus $\xi^{(i)}$ is length 30.
Then, feature vector $F^{(i)}$ is derived from $\xi^{(i)}$. It has joint velocities and joint accelerations, as well as joint positions.

To learn the temporal coherence between the actions, we deal with only the human action sequences $\lbrace \mathbf{a}^{h,(i)} \rbrace _{i=1}^N $. 
Based on the progress state representation, for any time frame $s$, the prefix sequence of $\mathbf{a}^{h,(i)}$ of length $s$ yields a progress state $\mathbf{p}_s^{(i)}$ and the current action $c_s^{h,(i)} = \mathbf{a}_s^{h,(i)}$.
The next action label $n_s^{h,(i)}$ performed after frame $s$ can also be computed, at which the action label differs at the first time while searching in the increasing order from frame $s+1$.
Then, for all possible pairs of demonstrations and frame index $(i, s)$, the tuples $(\mathbf{p}_s^{(i)}, c_s^{h,(i)}, n_s^{h,(i)})$ are collected to compute histograms $h(n^h; \mathbf{p}, c^h)$, which counts the next action labels at each pair $(\mathbf{p}, c^h)$ that have appeared at least once.
We use the normalized histograms to estimate the next future action for the given $\mathbf{p}$ and $c^h$. i.e.,
\begin{equation}
p(n^h = a_j^h | \mathbf{p}, c^h) = \frac{h(a_j^h; \mathbf{p}, c^h)}{\sum\limits_{k=1}^m h(a_k^h; \mathbf{p}, c^h)} . \label{eq:next_action}
\end{equation}

In the worst case, there are at most $O(2^{m})$ progress states since there are $m$ binary values per action.
However, in practice, 
only $O(N \cdot m)$ progress states are generated. It is because the number of unique progress states are less than $m$, and the subtask order dependency may allow only a few possible topological orders.

In order to train the human motion, the motion sequence $\xi^{(i)}$ and the feature sequence $F^{(i)}$ are learned, as well as the action sequences $\mathbf{a}^{(i)}$.
Because we are interested in short-term human motion prediction for collision avoidance, we train the learning module from multiple short periods of motion.
Let $n_p$ be the number of previous consecutive frames and $n_f$ be the number of future consecutive frames to look up.
$n_f$ and $n_p$ are decided so that the length of motion is short-term motion (about 1 second) that will be used for short-term future collision detection in the robot motion planner.
At the time frame $s$, where $n_p \leq s \leq T^{(i)} - n_f$, the columns of feature matrix $F^{(i)}$ from column index $s - n_p + 1$ to $s$ are denoted as $F_{prev,s}^{(i)}$. Similarly, the columns of motion matrix from index $s+1$ to $s + n_f$ are denoted as $\xi_{next,s}^{(i)}$.

Tuples $(F_{prev,s}^{(i)}, \mathbf{p}_s^{(i)}, c_s^{h(i)}, \xi_{next,s}^{(i)})$ for all possible pairs of $(i, s)$ are collected as the training input.
They are partitioned into groups having the same progress state $\mathbf{p}$.
For each progress state $\mathbf{p}$ and current action $c^h$, the set of short-term motions are regressed using SPGP with the DTW kernel function~\cite{muller2007dynamic}, considering $\lbrace F_{prev} \rbrace$ as input and $\lbrace \xi_{next} \rbrace$ as multiple channeled outputs.
We use SPGPs, a variant of Gaussian Processes, because it significantly reduces the running time for training and inference by choosing $M$ pseudo-inputs from a large number of an original human motion inputs.
The final learned probability distribution is
\begin{align}
p(\xi_{next} | F_{prev}, \mathbf{p}, c^h) &= \prod_{c : \text{channels}} p(\xi_{next,c} | F_{prev}, \mathbf{p}, c^h), \nonumber \\
p(\xi_{next,c} | F_{prev}, \mathbf{p}, c^h) & \sim \mathcal{GP}(m_{c}, K_{c}) \label{eq:GPRlearning} ,
\end{align}
where $\mathcal{GP}(\cdot, \cdot)$ represents trained SPGPs, $c$ is an output channel (i.e., an element of matrix $\xi_{next}$), and $m_{c}$ and $K_{c}$ are the learned mean and covariance functions of the output channel $c$, respectively.

The current action label $c_s^{h,(i)}$ should be learned to estimate the current action.
We train $c_s^{h,(i)}$ using Tuples $(F_{prev,s}^{(i)}, \mathbf{p}_s^{(i)}, c_s^{h(i)})$.
For each state $\mathbf{p}$, we use Import Vector Machine (IVM) classifiers to compute the probability distribution:
\begin{align}
p(c^h = a_j^h | F_{prev}, \mathbf{p}) = \frac{\exp(f_j(F_{prev}))}{\sum_{a_k^h} \exp(f_k(F_{prev}))} \label{eq:IVMclassifier} ,
\end{align}
where $f_j(\cdot)$ is the learned predictive function~\cite{zhu2012kernel} of IVM.

\subsection{Runtime human intention and motion inference}
\label{subsec:runtime}

Based on the learned human actions, at runtime we infer the next most likely short-term human motion and human subtask for the purpose of collision avoidance and task planning, respectively.
The short-term future motion prediction is used for the collision avoidance during the motion planning.
The probability of future motion is given as:
\begin{align*}
p(\xi_{next} | F, \mathbf{p}) = \sum_{c^h \in A^h} p(\xi_{next}, c^h | F, \mathbf{p}).
\end{align*}
By applying the Bayes theorem, we get
\begin{align}
p(\xi_{next} | F, \mathbf{p}) = \sum_{c^h \in A^h} p(c^h | F, \mathbf{p}) p(\xi_{next} | F, \mathbf{p}, c^h) \label{eq:motion_prediction} .
\end{align}
The first term $p(c^h | F, \mathbf{p})$ is inferred through the IVM classifier in (\ref{eq:IVMclassifier}).
To infer the second term, we use the probability distribution in (\ref{eq:GPRlearning}) for each output channel.

We use Q-learning for training the best robot action policy $\pi$ in our MDP-based task planner.
We first define the function $Q : P \times A^r \rightarrow \rm I\!R$, which is iteratively trained with the motion planning executions.
$Q$ is updated as
\begin{align*}
Q_{t+1}(\mathbf{p}_t, a_t^r) =& (1 - \alpha_t) Q_t(\mathbf{p}_t, a_t^r) \\
&+ \alpha_t (R_{t+1} + \gamma \max_{a^r} Q_t(\mathbf{p}_{t+1}, a^r) ) ,
\end{align*}
where the subscripts $t$ and $t+1$ are the iteration indexes, $R_{t+1}$ is the reward function after taking action $a_t^r$ at state $\mathbf{p}_t$, and $\alpha_t$ is the learning rate, where we set $\alpha_t = 0.1$ in our experiments.
A reward value $R_{t+1}$ is determined by several factors:
\begin{itemize}
\item Preparation for next human action: 
      the reward gets higher when the robot runs an action before a human action which can be benefited by the robot's action.
      We define this reward as $R_{prep}(\mathbf{p}_t, a_t^r)$.
      Because the next human subtask depends on the uncertain human decision, we predict the likelihood of the next subtask from the learned actions in (\ref{eq:next_action}) and use it for the reward computation.
      The reward value is given as
      \begin{equation}
      R_{prep}(\mathbf{p}_t, a_t^r) = \sum_{a^h \in A^h} p(n^h = a^h | \mathbf{p}_t) H(a^h, a_t^r) ,
      \label{eq:reward}
      \end{equation}
      where $H(a^h, a^r)$ is a prior knowledge of reward, representing the amount of how much the robot helped the human by performing the robot action $a^r$ before the human action $a^h$.
      If the robot action $a^r$ has no relationship with $a^h$, the $H$ value is zero.
      If the robot helped, the $H$ value is positive, otherwise negative.
\item Execution delay: 
      There may be a delay in the robot motion's execution due to the collision avoidance with the human.
      To avoid collisions, the robot may deviate around the human and make it without delay.
      In this case the reward function is not affected, i.e. $R_{delay,t} = 0$.
      However, there are cases that the robot must wait until the human moves to another pose because the human can block the robot's path, which causes delay $d$.
      We penalize the amount of delay to the reward function, i.e. $R_{delay,t} = -d$.
      Note that the delay can vary during each iteration due to the human motion uncertainty.
\end{itemize}
The total reward value is a weighted sum of both factors:
\begin{align*}
R_{t+1} =& w_{prep} R_{prep}(\mathbf{p}_t, a_t^r) + w_{delay} R_{delay,t}(\mathbf{p}_t, a_t^r) ,
\end{align*}
where $w_{prep}$ and $w_{delay}$ are weights for scaling the two factors.
The preparation reward value is predefined for each action pairs.
The delay reward is measured during runtime.

%% file: 5.tex
\section{I-Planner: Intention-aware motion planning}
\label{sec:planning}

Out motion planner is based on an optimization formulation, where $n+1$ waypoints in the space-time domain $\mathbf{Q}$ define a robot motion trajectory to be optimized.
Specifically, we use an optimization-based planner, ITOMP~\cite{Park:2012:ICAPS}, that repeatedly refines the trajectory while interleaving the execution and motion planning for dynamic scenes.
We handle three types of constraints: smoothness constraint, static obstacle collision-avoidance, and dynamic obstacle collision avoidance.
To deal with the uncertainty of future human motion, we use probabilistic collision detection between the robot and the predicted future human pose.

Let $s$ be the current waypoint index, meaning that the motion trajectory is executed in the time interval $[t_0, t_s]$, and let $m$ be the replanning time step.
A cost function for collisions between the human and the robot can be given as:
\begin{align}
\sum_{i=s+m}^{s+2m} p \left( \bigcup_{j,k} B_{jk}(\mathbf{q}_i) \cap C_{dyn}(t_i) \neq \emptyset \right) \label{eq:collision_probability}
\end{align}
where $C_{dyn}(t)$ are the workspace volumes occupied by dynamic human obstacles at time $t$.
The trajectory being optimized during the time interval $[t_s, t_{s+m}]$ is executed during the next time interval $[t_{s+m}, t_{s+2m}]$.
Therefore, the future human poses are considered only in the time interval $[t_{s+m}, t_{s+2m}]$.

The collision probability between the robot and the dynamic obstacle at time frame $i$ in (\ref{eq:collision_probability}) can be computed as a maximum between bounding spheres:
\begin{align}
\max_{j,k,l} p \left( B_{jk}(\mathbf{q_i}) \cap C_{l}(t_i) \neq \emptyset \right), \label{eq:max_collision_probability}
\end{align}
where $C_l(t_i)$ denotes bounding spheres for a human body at time $t_i$ whose centers are located at line segments between human joints.
The future human poses $\xi_{future}$ are predicted in (\ref{eq:motion_prediction}) and the bounding sphere locations $C_l(t_i)$ are derived from it.
Note that the probabilistic distribution of each element in $\xi_{future}$ is a linear combination of current action proposal $p(c^h | F, \mathbf{p})$ and Gaussians $p(\xi_{future} | F, \mathbf{p}, c^h)$ over all $c^h$, i.e., (\ref{eq:max_collision_probability}) can be reformulated as
\begin{align*}
\max_{j,k,l} \sum_{c_h} p(c^h | F, \mathbf{p}) p \left( B_{jk}(\mathbf{q_i}) \cap C_{l}(t_i) \neq \emptyset \right) .
\end{align*}
Let $\mathbf{z}_l^1$ and $\mathbf{z}_l^2$ be the probability distribution functions of two adjacent human joints obtained from $\xi_{future}(t_i)$, where the center of $C_l(t_i)$ is located between them by a linear interpolation $C_l(t_i) = (1-u) \mathbf{z}_l^1 + u \mathbf{z}_l^2$ where $0 \leq u \leq 1$.
The joint positions follows Gaussian probability distributions:
\begin{align}
\mathbf{z}_l^i &\sim \mathcal{N}(\mu_l^i, \Sigma_l^i) \nonumber \\
\mathbf{c}_l(t_i) &\sim \mathcal{N}((1-u) \mu_l^1 + u \mu_l^2, (1-u)^2 \Sigma_l^1 + u^2 \Sigma_l^2) \\
&= \mathcal{N}(\mu_l, \Sigma_l) \label{eq:center_distribution} ,
\end{align}
where $\mathbf{c}_l(t_i)$ is the center of $C_l(t_i)$.
Thus, the collision probability between two bounding spheres is bounded by
\begin{align}
\int_{\rm I\!R^3} I( || \mathbf{x} - \mathbf{b}_{jk}(\mathbf{q}_i) ||^2 \leq (r_1 + r_2)^2 ) f(\mathbf{x}) d\mathbf{x} ,
\label{eq:sphere_collision}
\end{align}
where $\mathbf{b}_{jk}(\mathbf{q}_i)$ is the center of bounding sphere $B_{jk}(\mathbf{q}_i)$,
$r_1$ and $r_2$ are the radius of $B_{jk}(\mathbf{q}_i)$ and $C_l(t_i)$, respectively, $I(\cdot)$ is an indicator function, and $f(\mathbf{x})$ is the probability distribution function. The indicator function restricts the integral domain to a solid sphere, and $f(\mathbf{x})$ is the probability density function of $\mathbf{c}_l(t_i)$, in (\ref{eq:center_distribution}).
There is no closed form solution for (\ref{eq:sphere_collision}), therefore we use the maximum possible value to approximate the probability.
We compute $\mathbf{x}_{max}$ at which $f(\mathbf{x})$ is maximized in the sphere domain and multiply it by the volume of sphere, i.e.
\begin{align}
p \left( B_{jk}(\mathbf{q_i}) \cap C_{l}(t_i) \neq \emptyset \right) \leq \frac{4}{3} \pi (r_1 + r_2)^2 f(\mathbf{x}_{max}) \label{eq:approx_sphere_collision} .
\end{align}
Since even $\mathbf{x}_{max}$ does not have a closed form solution, we use the bisection method to find $\lambda$ with 
\begin{align*}
\mathbf x_{max}&=(\Sigma^{-1}+\lambda I)^{-1}(\Sigma^{-1}\mathbf p_{lm}+\lambda \mathbf o_{jk}(\mathbf q_i)),
\end{align*}
which is on the surface of sphere, explained in Generalized Tikhonov regularization~\cite{groetsch1984theory} in detail.

The collision probability, computed in (\ref{eq:sphere_collision}), is always positive due to the uncertainty of the future human position, and we compute a trajectory that is guaranteed to be nearly collision-free with sufficiently low collision probability.
For a user-specified confidence level $\delta_{CL}$, we compute a trajectory that its probability of collision is upper-bounded by $(1 - \delta_{CL})$.
If it is unable to compute a collision-free trajectory, a new waypoint $\mathbf{q}_{new}$ is appended next to the last column of $\mathbf{Q}$ to make the robot wait at the last collision-free pose until it finds a collision-free trajectory.
This approach computes a guaranteed collision-free trajectory, but leads to delay, which is fed to the Q-learning algorithm for the MDP task planner.
The higher the delay that the collision-free trajectory of a task has, the less likely the task planner selects the task again.

%% file: 6.tex
\section{Analysis}
\label{sec:analysis}

The overall performance is governed by three factors:
the predicted human motions described in Section~\ref{sec:prediction};
The optimization-based motion planner described in Section~\ref{sec:planning}; and the collision probability between the  robot and  predicted human motions for safe trajectory computation.
In this section, we analyze the performance and accuracy of each factor.

\subsection{Human motion prediction with noisy input}
\label{subsec:noisy_input}
In most scenarios, our human motion prediction algorithm is dealing with the noisy data. As a result, it is important analyze the performance of our approach taking into account these limitations.
To analyze the robustness of our human motion prediction algorithm, we take into account input noise in our Gaussian Process model.

Equation \ref{eq:GPRlearning} is the Gaussian Process Regression for human motion prediction, where the input variable is $F_{prev}$ and the output variables are represented as$\xi_{next,c}$.
To follow the standard notation of Gaussian Process, we use the symbol $\mathbf{x}$ as a $\mathit{D}$-dimensional input vector instead of $F_{prev}$, $y$ as an output variable instead of $\xi_{next,c}$, and $y = f(\mathbf{x}) + \epsilon_y$ instead of $p(\xi_{next,c} | F_{prev}) ~ \mathcal{GP}(m_c, K_c)$.
We add an input noise term $\mathbf{\epsilon}_x$ to the standard GP model,
\begin{align*}
y = \tilde{y} + \epsilon_y, & \, \epsilon_y \sim \mathcal{N}(0, \sigma_y^2),\\
\mathbf{x} = \mathbf{\tilde{x}} + \mathbf{\epsilon_x}, & \, \mathbf{\epsilon}_x \sim \mathcal{N}(\mathbf{0}, \Sigma_x),
\end{align*}
where we assume that the input noise is Gaussian and the $D$-dimensional input vector is independently noised, i.e. $\Sigma_x$ is diagonal.
With the input noise term, the output becomes
\begin{align*}
y = f(\tilde{\mathbf{x}} + \mathbf{\epsilon}_x) + \epsilon_y,
\end{align*}
and the first term Taylor expansion on the function $f$ yields
\begin{align*}
y = f(\tilde{\mathbf{x}}) + \mathbf{\epsilon}_x^T \mathbf{\partial}_f(\mathbf{x}) + \epsilon_y,
\end{align*}
where $\mathbf{\partial}_f(\mathbf{x})$ is the $D$-dimensional derivative of $f$ with respect to $\mathbf{x}$.
We have $N$ training data items, represented as $(\mathbf{x}_i, y_i)_{i=1}^N$.
$\mathbf{y}$ is a $N$-dimensional vector $\lbrace y_1, \, \cdots, y_N \rbrace^T$.
Following the derivation of the Gaussian Process(GP) with the additional error term presented in~\cite{mchutchon2011gaussian}, GP with noisy input becomes
\begin{align*}
m_c(\mathbf{x}_*) &= \mathbf{k}_{*} \left( K + \sigma_y^2 I + \mathrm{diag} \left( \Delta_f \Sigma_x \Delta_f^T \right) \right)^{-1} \mathbf{y}, \\
K_c(\mathbf{x}_*) &= k_{**} - \mathbf{k}_{*} \left( K + \sigma_y^2 I + \mathrm{diag} \left( \Delta_f \Sigma_x \Delta_f \right) \right)^{-1} \mathbf{k}_{*},
\end{align*}
where $\mathbf{x}_*$ is the input of mean function $m_c$ and variance function $K_c$, $k_{**}$ is the kernel function value on $\mathbf{x}_*$, i.e. $k_{**} = k(\mathbf{x}_*, \mathbf{x}_*)$, $K$ is a matrix of kernel function values on all pairs of input points, i.e. $K_{ij} = k(\mathbf{x}_i, \mathbf{x}_j)$, and $\Delta_f$ is a matrix whose $i$-th row is the derivative of $f$ at $\mathbf{x}_i$.
$\mathrm{diag}(\cdot)$ results in a diagonal matrix, leaving the diagonal elements and eliminating the non-diagonals to zero.
Compared to the standard GP, the additional term is the diagonal matrix $\mathrm{diag} \left( \Delta_f \Sigma_x \Delta_f \right)$, acting as the output noise term $\sigma_y I$.

In our human motion prediction algorithm, we use the human joint positions as input and output, so the input and the output have the same amount of noise.
In other words, $\Sigma_x = \sigma_x I$ and $\sigma_x = \sigma_y$.
Instead of differentiating the input and output noise, we use $\sigma = \sigma_x = \sigma_y$.
Also, the derivative $\mathbf{\partial}_f$ is proportional to the joint velocities, because $f$ is proportional to the joint position values.
Therefore, the elements of the diagonal matrix can be expressed as
\begin{align*}
\left( \mathrm{diag} \left( \Delta_f \Sigma_x \Delta_f \right) \right)_{ii} = \sigma_x || \mathbf{\partial}_f (\mathbf{x}_i) ||^2 = \sigma || \mathbf{v}_i ||^2,
\end{align*}
where $\mathbf{v}_i$ is the joint velocity.
Because the joint velocity is a variable for every training input data, instead of taking joint velocities, we set the joint velocity limits $\mathbf{v}$, satisfying $|| \mathbf{v}_i ||^2 \leq || \mathbf{v} ||^2$.
As a result, the Gaussian Process regression for human motion prediction can be given as:
\begin{align*}
m_c(\mathbf{x}_*) &= \mathbf{k}_{*} \left( K + \sigma^2 \left( 1 + || \mathbf{v} ||^2 \right) I\right)^{-1} \mathbf{y}, \\
K_c(\mathbf{x}_*) &= k_{**} - \mathbf{k}_{*} \left( K + \sigma_y^2 \left( 1 + || \mathbf{v} ||^2 \right) I \right)^{-1} \mathbf{k}_{*}.
\end{align*}
As long as we keep the joint velocity small, the Gaussian Process with input noise behaves robustly, similar to how it behaves without input noise.
The mean function is over-smoothed and the variance function becomes higher if the joint velocity limit is set high.

\subsection{Upper bound of collision probability}
 Using the predicted distribution and user-specified threshold $\delta_{CL}$, we can compute an upper bound using the following lemma.

\begin{lemma}
The collision probability is less than $(1 - \delta_{CL})$ if $\frac{4}{3} \pi (r_1+r_2)^2 f(\mathbf{x}_{max}) < 1 - \delta_{CL}$.
\end{lemma}

\begin{proof}
This bounds follows Equation (\ref{eq:sphere_collision}) and (\ref{eq:approx_sphere_collision}).
\begin{align*}
    p \left( B_{jk}(\mathbf{q_i}) \cap C_{l}(t_i) \neq \emptyset \right) & \leq \frac{4}{3} \pi (r_1 + r_2)^2 f(\mathbf{x}_{max}) \\
     & < 1 - \delta_{CL} .
\end{align*}
\end{proof}
We use this bound in our optimization algorithm for collision avoidance.

\subsection{Safe trajectory optimization}
Our goal is to compute a robot trajectory that will either not collide with  the human or reduces the probability of collision below a certain threshold. Sometimes, there is no feasible collision-free trajectory that the robot cannot avoid  with the human.
However, if there is any trajectory where the collision probability  is less than a threshold, we seek to compute such a trajectory.
Our optimization-based planner also generates multiple initial trajectories and finds the best solution in parallel. In this manner, it expands the search space and reduces the probability of the robot being stuck in a local minima configuration. 
This can be expressed as the following theorem:
\begin{theorem}
An optimization-based planner with $n$ parallel threads will compute the a global solution trajectory with collision probability less than $(1 - \delta_{CL})$, with the probability $1-(1-\frac{|A(\delta_{CL})|}{|S|})^n$, if it exists, where $S$ is the entire search space. $A(\delta_{CL})$ corresponds to the neighborhood around the optimal solutions, with collision probability being less than $(1 - \delta_{CL})$, where the local optimization converges to one of the global optima. $|\cdot|$ is the measure of the search or configuration space.
\label{theorem:threads}
\end{theorem}

We give a brief overview of the proof. In this case, $\frac{|A(\delta_{CL})|}{|S|}$ measures the probability that a random sample lies in the neighborhood of the global optima. 
In general, $|A(\delta_{CL})|$ will be smaller as the environment becomes more complex and has more local minima.  
Overall, this theorem provides a lower bound on
the probability that our intention-aware planner with $n$ threads.
In the limit that $n$ increases, the planner will compute the optimal solution if it exist. This can be stated as:

\begin{corollary}[Probabilistic Completeness]
Our intention-aware motion planning algorithm with $n$ trajectories is probabilistic complete, as $n$ increases.
\end{corollary}
\begin{proof}
$$\lim_{n\to\infty} 1 - \left( 1-\frac{|A(\delta_{CL})|}{|S|} \right)^n = 1 .$$
\end{proof}

%% file: 7.tex
\section{Implementation and performance}
\label{sec:results}

We highlight the performance of our algorithm in a situation where the robot is performing a collaborative task with a human and computing safe trajectories.
We use a 7-DOF KUKA-IIWA robot arm.
The human motion is captured by a Kinect sensor operating with a 15Hz frame rate, and only upper body joints are tracked for collision checking.
We use the ROS software~\cite{quigley2009ros} for robot control and sensor data communication.
The motion planner has a 0.5s re-planning timestep, 
The number of pseudo-inputs $M$ of SPGPs is set to $100$ so that the prediction computation is performed online.

\subsection{Performance of human motion prediction}

\begin{figure*}[ht]
  \centering
  \begin{subfigure}[]{0.33\linewidth}
    \centering
    \includegraphics[width=\textwidth]{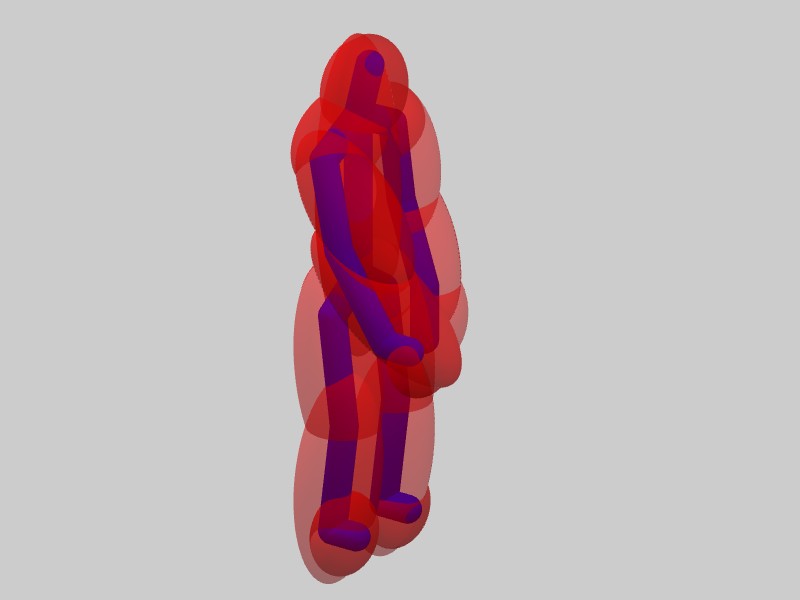}
    \caption{}
  \end{subfigure}
  \begin{subfigure}[]{0.33\linewidth}
    \centering
    \includegraphics[width=\textwidth]{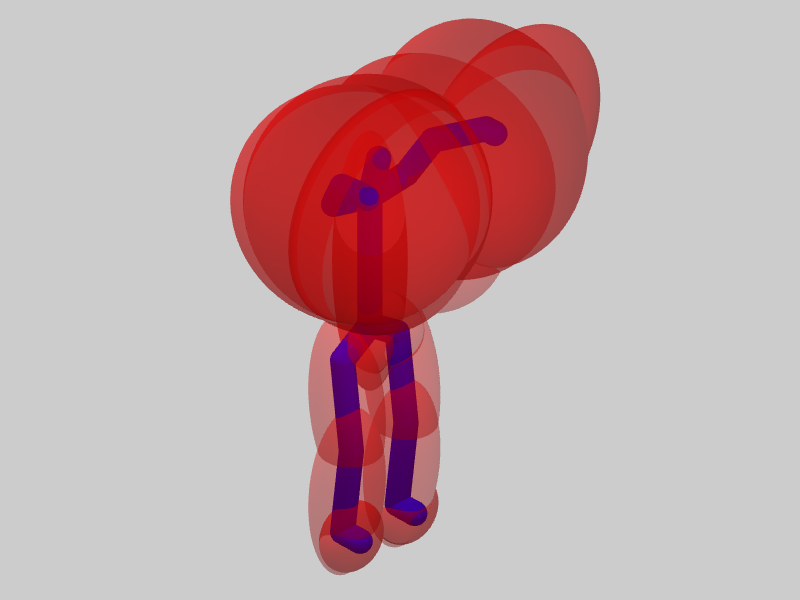}
    \caption{}
  \end{subfigure}
  \caption{
           {\bf Human motion prediction results:} The result of human motion prediction is represented by Gaussian distributions for each skeleton joint. The ellipsoid boundaries within which the integral of Gaussian distribution probability is 95\% are drawn in red, which the human skeleton is shown in blue.  Bounding ellipsoids have transparency values that are proportional to the action classifier probability.
           {
           (a) Undistinguished human action class: This occurs when the classifier fails to distinguish the human action class, and thereby generating nearly uniform probability distribution among the action classes.
           (b) Prediction results when untrained human motion is given: These cases result in larger boundary spheres around the human skeleton. This is due to the reason that in  the Gaussian Process, the output has a uniform mean and a high variance when the input point is outside the range of the training input data.
           }} \label{fig:image_prediction}
\end{figure*}

\begin{figure*}[ht]
  \centering
  \begin{subfigure}[]{0.32\linewidth}
    \centering
    \includegraphics[width=\textwidth]{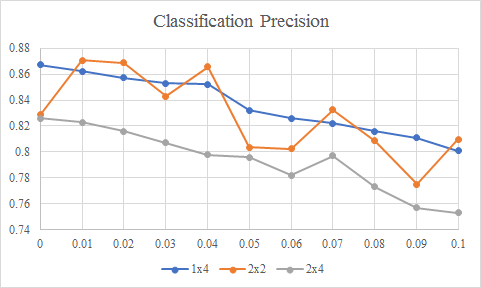}
    \caption{}
  \end{subfigure}
  \begin{subfigure}[]{0.32\linewidth}
    \centering
    \includegraphics[width=\textwidth]{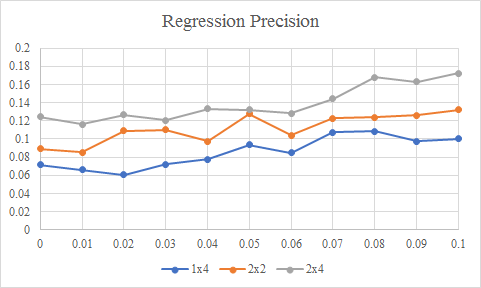}
    \caption{}
  \end{subfigure}
  \begin{subfigure}[]{0.32\linewidth}
    \centering
    \includegraphics[width=\textwidth]{figs/prediction_result1.png}
    \caption{}
  \end{subfigure}
  \caption{
           {\bf Performance of human motion prediction:} The precision/accuracy of classification and regression algorithms are shown in the graphs, with varying input noise.
           {
           (a) Classification precision versus input noise: We only take the input points where the probability of an action classification is dominant, i.e. probability $> 50\%$.
           (b) Regression precision versus input noise: It is measured by the integral of distance between the predicted human motion and the ground-truth human motion trajectory.
           (c) Regression precision versus input noise: It is the integral of the volume of Gaussian distribution ellipsoids.
           }} \label{fig:result_prediction}
\end{figure*}

We have tested our human motion prediction model on labeled motion datasets that correspond to a human reaching action.
Our human motion prediction model allows human joint position errors, as discussed in the prior Section.
In order to demonstrate the robustness of our model against input noise errors, we measure the classification accuracy and regression accuracy, varying the sensor error parameter and the maximum human joint velocity limits.

In the human reaching action motion datasets, the human is initially in a static pose in front of a table.
Then, a left or right arm moves and reaches towards one of the target locations on the table, and then returns to the initial pose.
The dataset contains $8$ different target positions on the table, and $30$ reaching motions for each target position.
Because the result of our human motion prediction model depends on human joint velocities, we synthesize fast and slow motions by changing the speed of captured motions.

We measure the correctness of human motion classification and human motion regression in the following ways:
\begin{itemize}
\item Motion classification precision: Since human motion is continuous and the transition between motions can be hard to measure, we only count the motion frames where the multi-class classifier yields the highest probability that is greater than $50\%$.

\item Motion regression precision: At a certain time, human motion trajectory for the next $T$ seconds has been predicted. We compute the regression precision as the integral of distance between the predicted and ground-truth motions over the $T$-second window as:
\begin{align*}
\int_{0}^{T} \sum_{i : joint} || \mathbf{p}_{pred,i}(t) - \mathbf{p}_{true,i}(t) ||^2 \, dt ,
\end{align*}
where $\mathbf{p}_{pred}(t)$ is the collection of resulting mean values of the Gaussian Process with noisy input for joint $i$. 

\item Motion regression accuracy: Similar to the precision, the accuracy is the integral of the volume of ellipsoids generated by the Gaussian distributions:
\begin{align*}
\int_{0}^{T} \sum_{i : joint} \det \left( \mathbf{K}_{i,pred}(t) \right) \, dt ,
\end{align*}
where $\mathbf{K}_{i,pred}(t)$ is the collection of resulting variance of the Gaussian Process with noisy input for joint $i$.
In this case, a lower value results in a better prediction result.
\end{itemize}

Figure~\ref{fig:image_prediction} shows the human motion prediction results.
The Gaussian Process regression algorithm corresponds to using Gaussian distribution ellipsoids around the predicted mean values of the joint positions.
In (a), when the human is in the idle position, the motion classifier results in near uniform distribution among the action classes, and the motion regression algorithm generates human motion trajectories that progress slightly towards the target positions.
In (b), when the human arm is moving close to the target position, the motion classifier predicts the motion class with a dominant probability, and the motion regression algorithm infers a correct future motion trajectory, that is more accurate than (a).
In (c), when an untrained human motion is given, the motion classifier and the motion regression algorithm result in uniform values and high variances, respectively, generating a conservative collision bound around the human.
This is a normal behavior in terms of classification and regression, because the given input point is outside the range of training data input points.

Figure~\ref{fig:result_prediction} shows the classification precision, regression precision and accuracy with varying input noises.
When the input noise is high, the accuracy of human motion classification and regression can go down.

\subsection{Robot motion planning with human motion prediction}

In the simulated benchmark scenario, the human is sitting in front of a desk. In this case, the robot arm helps the human by delivering objects from one position that is far away from the human to target position closer to the human. The human waits till the robot delivers the object. As different tasks are performed in terms of picking the objects and their delivery to the goal position, the temporal coherence is used to predict the actions.
The action set for a human is $A^h = \lbrace \mathit{Take_0}, \mathit{Take_1}, \cdots \rbrace$, where $\mathit{Take_i}$ represents an action of taking object $i$ from its current position to the new position.
The action set for the robot arm is defined as $A^r = \lbrace \mathit{Fetch_0}, \mathit{Fetch_1}, \cdots \rbrace$.

We quantitatively measure the following values:
\begin{itemize}
\item Modified Hausdorff Distance (MHD)~\cite{dubuisson1994modified}: The distance between the ground-truth human trajectory and the predicted mean trajectory is used. In our experiments, MHD is measured for an actively moving hand joint over 1 second. This evaluates the quality of the anticipated trajectory of the human motion.

\item Smoothness: We also measured the smoothness of the robot's trajectory with and without human motion prediction results. The smoothness is computed as
\begin{align}
\frac{1}{T} \int_{0}^{T} \sum_{i=1}^{n} \ddot{\mathbf{q}}_{i}(t)^{2} \, dt , \label{eq:smoothness}
\end{align}
where the two dots indicate acceleration of joint angles.

\item Jerkiness: It is defined as
\begin{align}
\max\limits_{0 \leq t \leq T} \sum_{i=1}^{n} \ddot{\mathbf{q}}_{i}(t)^{2} . \label{eq:jerkiness}
\end{align}

\item Distance: The closest distance between the robot and the human during task collaboration.

\item Efficiency: This is the ratio of the task completion time when the robot and the human collaborate to complete all the subtasks, to the task completion time when the human performs all the tasks without any help from the robot. This is used to evaluate the capability of the resulting human-robot system.
\end{itemize}

To compare the performance of our I-Planner, we use the following algorithm:
\begin{itemize}
\item \textbf{ITOMP}. This model is the same as the realtime motion planner for dynamic environments, ITOMP~\cite{Park:2012:ICAPS}, without human motion prediction.
\item \textbf{I-Planner, no noisy input (I-Planner, no NI)}. This is our motion planning algorithm with  human motion prediction, but the motion prediction does not assume noisy input. The details of this algorithm are given in the preliminary paper~\cite{park2017intention}.
\item \textbf{I-Planner, noisy input (I-Planner, NI)}. This is the modified algorithm presented in the previous section that also takes into account noise in the human motion prediction data.
\end{itemize}

Table~\ref{table:performance} highlights the performance of our algorithm in three different variations of this scenario: \textit{arrangements of blocks}, \textit{task order} and \textit{confidence level}.
The numbers in the "Task Order" column indicates the identifiers of human actions.
Parentheses mean that the human actions in the parentheses can be performed in any order.
Arrows mean that the right actions can be performed only if the left actions are done.
For example, $(0, 1) \rightarrow (2, 3)$ means that the possible action orders are $0 \rightarrow 1 \rightarrow 2 \rightarrow 3$, $0 \rightarrow 1 \rightarrow 3 \rightarrow 2$, $1 \rightarrow 0 \rightarrow 2 \rightarrow 3$ and $1 \rightarrow 0 \rightarrow 3 \rightarrow 2$.
Table~\ref{table:performance_real} shows the performance of our algorithm with a real robot. Our algorithm has been implemented on a PC with 8-core i7-4790 CPU.
We used OpenMP to parallelize the computation of future human motion prediction and probabilistic collision checking.

\begin{table*}[ht]
\centering
\begin{tabular}{|c|c|c|c|c|}
\hline
\multirow{2}{*}{Scenarios} & \multirow{2}{*}{Arrangement} & \multirow{2}{*}{\begin{tabular}[x]{@{}c@{}}Task Order\end{tabular}}  & \multirow{2}{*}{\begin{tabular}[x]{@{}c@{}}Confidence \\ Level\end{tabular}} &  \multirow{2}{*}{\begin{tabular}[x]{@{}c@{}}Average \\ Prediction Time\end{tabular}}\\
& & & & \\ \hline
\multirow{3}{*}{\begin{tabular}[x]{@{}c@{}}Different\\Arrangements\end{tabular}} & $1 \times 4$ & $(0,1) \rightarrow (2,3)$ & 0.95 & 52.0 ms \\ \cline{2-5}
& $2 \times 2$ & $(1,5) \rightarrow (2,6)$ & 0.95 & 72.4 ms \\ \cline{2-5}
& $2 \times 4$ & $(0,4) \rightarrow (1,5) \rightarrow (2,6) \rightarrow (3,7)$ & 0.95 & 169 ms \\ \hline
\multirow{3}{*}{\begin{tabular}[x]{@{}c@{}}Temporal\\Coherence\end{tabular}} & $1 \times 4$ & $0 \rightarrow 1 \rightarrow 2 \rightarrow 3$ & 0.95 & 52.1 ms \\ \cline{2-5}
& $1 \times 4$ & Random & 0.95 & 105 ms \\ \cline{2-5}
& $1 \times 4$ & $(0,2) \rightarrow (1,3)$ & 0.95 & 51.7 ms \\ \hline
\multirow{3}{*}{\begin{tabular}[x]{@{}c@{}}Confidence\\Level\end{tabular}} & $1 \times 4$ & $0 \rightarrow (1,2) \rightarrow 3$ & 0.90 & 47.2 ms \\ \cline{2-5}
& $1 \times 4$ & $0 \rightarrow (1,2) \rightarrow 3$ & 0.95 & 50.7 ms \\ \cline{2-5}
& $1 \times 4$ & $0 \rightarrow (1,2) \rightarrow 3$ & 0.99 & 155 ms \\ \hline
\end{tabular}
\caption{Three different simulation scenarios: \textit{Different Arrangements}, \textit{Temporal Coherence}, and \textit{Confidence Level}. We take into account different arrangements of blocks as well as the confidence levels used for probabilistic collision checking. These confidence levels are used for motion prediction.}
\label{table:performance}
\end{table*}

\begin{table*}[ht]
\centering
\begin{tabular}{|c|c|c|c|c|c|c|c|}
\hline
Scenarios & Model & Prediction Time & MHD & Smoothness & Jerkiness & Distance & Efficiency \\ \hline
\multirow{3}{*}{\begin{tabular}[x]{@{}c@{}}Different\\Arrangements (1)\end{tabular}}
& ITOMP            & N/A     & N/A    & 2.96 & 5.23 & 3.2 cm & 1.2 \\ \cline{2-8}
& I-Planner, no NI & 52.0 ms & 6.7 cm & 1.08 & 1.52 & 6.7 cm & 1.6 \\ \cline{2-8}
& I-Planner, NI    & 58.5 ms & 7.2 cm & 0.92 & 1.03 & 8.1 cm & 1.7 \\ \hline
\multirow{3}{*}{\begin{tabular}[x]{@{}c@{}}Different\\Arrangements (2)\end{tabular}}
& ITOMP            & N/A     & N/A    & 5.78 & 7.19 &  2.3 cm & 1.1 \\ \cline{2-8}
& I-Planner, no NI & 72.4 ms & 6.2 cm & 1.04 & 1.60 &  8.2 cm & 1.6 \\ \cline{2-8}
& I-Planner, NI    & 70.8 ms & 7.0 cm & 0.84 & 1.32 & 10.7 cm & 1.6 \\ \hline
\multirow{3}{*}{\begin{tabular}[x]{@{}c@{}}Different\\Arrangements (3)\end{tabular}}
& ITOMP            & N/A    & N/A     & 4.82 & 6.82 & 1.6 cm & 1.2 \\ \cline{2-8}
& I-Planner, no NI & 169 ms & 10.4 cm & 1.15 & 1.30 & 6.2 cm & 1.6 \\ \cline{2-8}
& I-Planner, NI    & 150 ms & 12.6 cm & 1.02 & 1.20 & 8.9 cm & 1.6 \\ \hline
\multirow{3}{*}{\begin{tabular}[x]{@{}c@{}}Temporal\\Coherence (1)\end{tabular}}
& ITOMP            & N/A     & N/A    & 1.79 & 3.22 & 6.0 cm & 1.3 \\ \cline{2-8}
& I-Planner, no NI & 52.1 ms & 4.3 cm & 0.65 & 1.56 & 9.3 cm & 1.5 \\ \cline{2-8}
& I-Planner, NI    & 48.9 ms & 5.2 cm & 0.62 & 1.53 & 9.7 cm & 1.5 \\ \hline
\multirow{3}{*}{\begin{tabular}[x]{@{}c@{}}Temporal\\Coherence (2)\end{tabular}}
& ITOMP            & N/A    & N/A    & 5.49 & 7.30 & 2.0 cm & 1.0 \\ \cline{2-8}
& I-Planner, no NI & 105 ms & 8.2 cm & 1.21 & 1.28 & 8.8 cm & 1.6 \\ \cline{2-8}
& I-Planner, NI    & 110 ms & 9.9 cm & 1.08 & 1.10 & 9.9 cm & 1.6 \\ \hline
\multirow{3}{*}{\begin{tabular}[x]{@{}c@{}}Temporal\\Coherence (3)\end{tabular}}
& ITOMP            & N/A     & N/A    & 3.12 & 3.18 &  8.0 cm & 1.3 \\ \cline{2-8}
& I-Planner, no NI & 51.7 ms & 6.8 cm & 1.00 & 1.20 & 12.1 cm & 1.5 \\ \cline{2-8}
& I-Planner, NI    & 60.0 ms & 8.2 cm & 0.79 & 0.93 & 13.2 cm & 1.5 \\ \hline
\multirow{3}{*}{\begin{tabular}[x]{@{}c@{}}Confidence\\Level (1)\end{tabular}}
& ITOMP            & N/A     & N/A    & 2.90 & 3.40 &  7.2 cm & 1.2 \\ \cline{2-8}
& I-Planner, no NI & 47.2 ms & 7.9 cm & 1.17 & 1.58 &  9.5 cm & 1.6 \\ \cline{2-8}
& I-Planner, NI    & 52.1 ms & 9.1 cm & 1.08 & 1.32 & 10.2 cm & 1.7 \\ \hline
\multirow{3}{*}{\begin{tabular}[x]{@{}c@{}}Confidence\\Level (2)\end{tabular}}
& ITOMP            & N/A     & N/A    & 3.12 & 3.80 & 10.3 cm & 1.2 \\ \cline{2-8}
& I-Planner, no NI & 50.7 ms & 7.9 cm & 1.28 & 1.71 & 13.3 cm & 1.6 \\ \cline{2-8}
& I-Planner, NI    & 48.2 ms & 9.1 cm & 1.20 & 1.66 & 14.1 cm & 1.8 \\ \hline
\multirow{3}{*}{\begin{tabular}[x]{@{}c@{}}Confidence\\Level (3)\end{tabular}}
& ITOMP            & N/A    & N/A    & 3.76 & 4.33 & 13.0 cm & 1.2 \\ \cline{2-8}
& I-Planner, no NI & 155 ms & 7.9 cm & 1.40 & 1.90 & 16.2 cm & 1.7 \\ \cline{2-8}
& I-Planner, NI    & 129 ms & 9.1 cm & 1.35 & 1.73 & 17.7 cm & 1.8 \\ \hline
\end{tabular}
\caption{Performance of the planner in robot motion simulation. The prediction results in a smoother trajectory and we observe up to 4X improvement in our smoothness metric defined in Equation (\ref{eq:smoothness}). The overall planner runs in real time.}
\label{table:performance_simulation}
\end{table*}

\begin{table*}[ht]
\centering
\begin{tabular}{|c|c|c|c|c|c|c|}
\hline
Scenarios & Model & Prediction Time & MHD & Smoothness & Jerkiness & Distance \\ \hline
\multirow{3}{*}{Waving Arms}
& ITOMP            & N/A     & N/a    & 4.88 & 6.23 &  2.1 cm \\ \cline{2-7}
& I-Planner, no NI & 20.9 ms & 5.0 cm & 0.91 & 1.33 &  9.3 cm \\ \cline{2-7}
& I-Planner, NI    & 23.5 ms & 6.1 cm & 0.83 & 1.25 & 10.5 cm \\ \hline
\multirow{3}{*}{Moving Cans}
& ITOMP            & N/A     & N/A    & 5.13 & 7.83 &  3.9 cm \\ \cline{2-7}
& I-Planner, no NI & 51.7 ms & 7.3 cm & 1.04 & 1.82 &  8.7 cm \\ \cline{2-7}
& I-Planner, NI    & 50.0 ms & 8.8 cm & 0.93 & 1.32 & 13.5 cm \\ \hline
\end{tabular}
\caption{Performance of the planner with a real robot running on the 7-DOF Fetch robot next to dynamic human obstacles. The online motion planner computes safe trajectories for challenging benchmarks like "moving cans." We observe almost 5X improvement in the smoothness of the trajectory due to our prediction algorithm.}
\label{table:performance_real}
\end{table*}

\subsection{Robot motion responses to human motion speed}

\begin{figure*}[ht]
  \centering
  \begin{subfigure}[t]{0.24\linewidth}
    \centering
    \includegraphics[width=\textwidth]{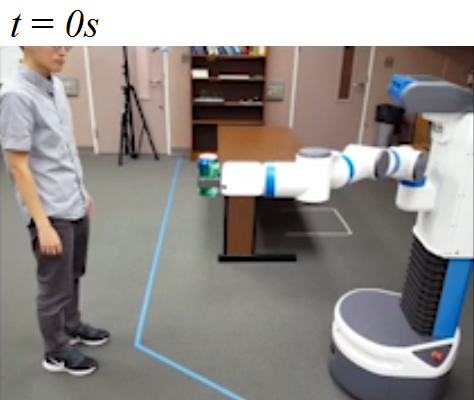}
  \end{subfigure}
  \begin{subfigure}[t]{0.24\linewidth}
    \centering
    \includegraphics[width=\textwidth]{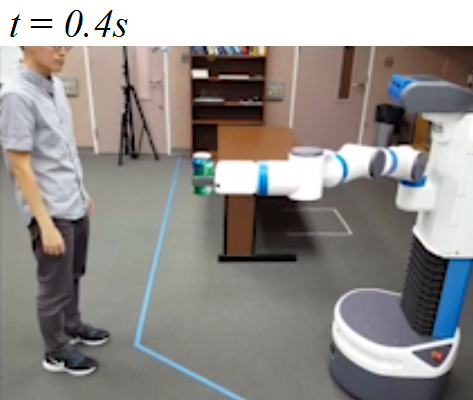}
  \end{subfigure}
  \begin{subfigure}[t]{0.24\linewidth}
    \centering
    \includegraphics[width=\textwidth]{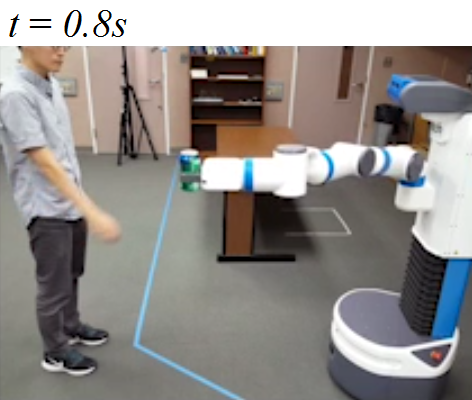}
    \caption{}
  \end{subfigure}
  \begin{subfigure}[t]{0.24\linewidth}
    \centering
    \includegraphics[width=\textwidth]{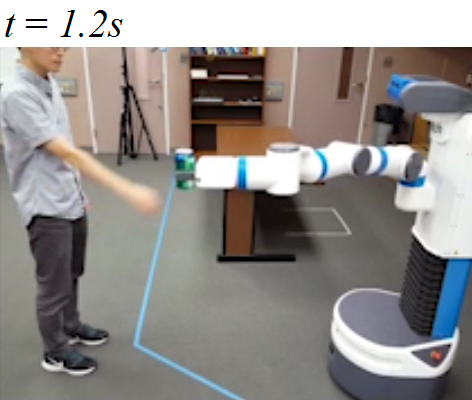}
  \end{subfigure}
  \\
  \begin{subfigure}[t]{0.24\linewidth}
    \centering
    \includegraphics[width=\textwidth]{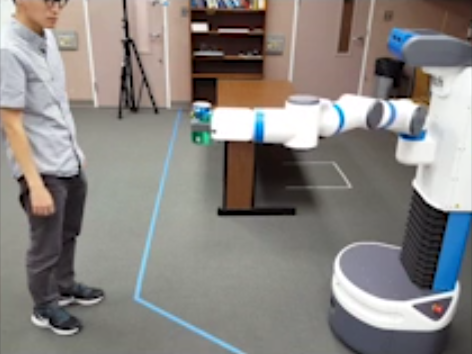}
  \end{subfigure}
  \begin{subfigure}[t]{0.24\linewidth}
    \centering
    \includegraphics[width=\textwidth]{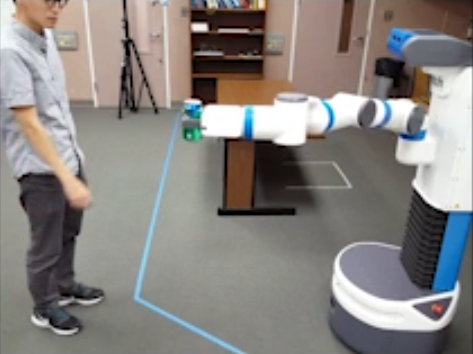}
  \end{subfigure}
  \begin{subfigure}[t]{0.24\linewidth}
    \centering
    \includegraphics[width=\textwidth]{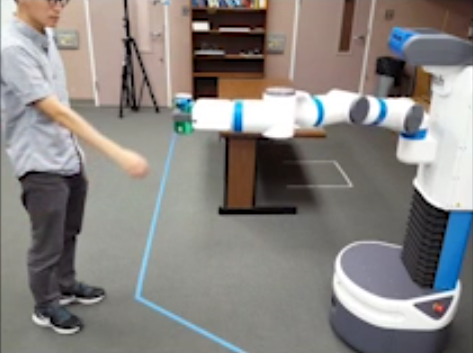}
    \caption{}
  \end{subfigure}
  \begin{subfigure}[t]{0.24\linewidth}
    \centering
    \includegraphics[width=\textwidth]{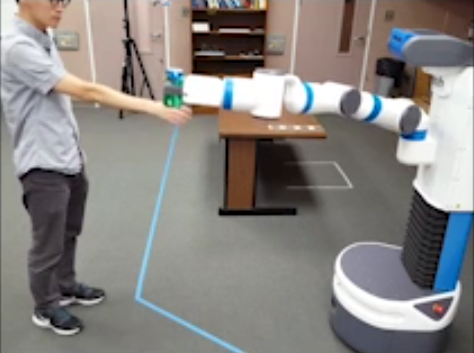}
  \end{subfigure}
  \\
  \begin{subfigure}[t]{0.24\linewidth}
    \centering
    \includegraphics[width=\textwidth]{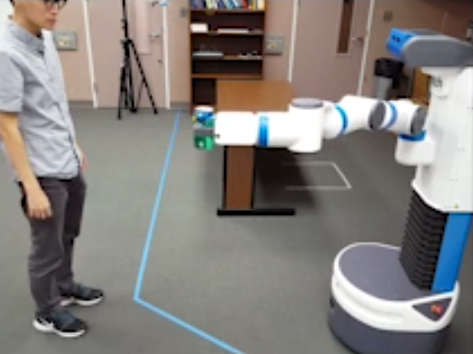}
  \end{subfigure}
  \begin{subfigure}[t]{0.24\linewidth}
    \centering
    \includegraphics[width=\textwidth]{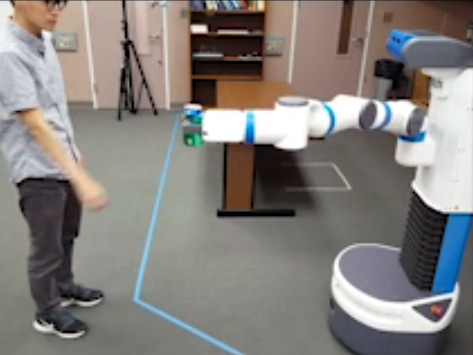}
  \end{subfigure}
  \begin{subfigure}[t]{0.24\linewidth}
    \centering
    \includegraphics[width=\textwidth]{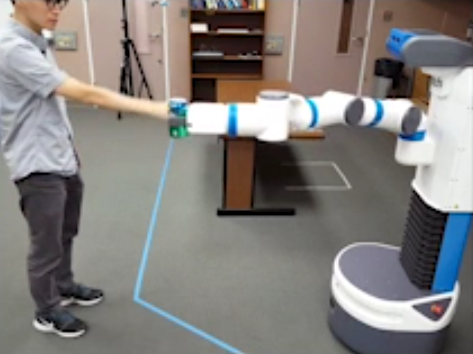}
    \caption{}
  \end{subfigure}
  \begin{subfigure}[t]{0.24\linewidth}
    \centering
    \includegraphics[width=\textwidth]{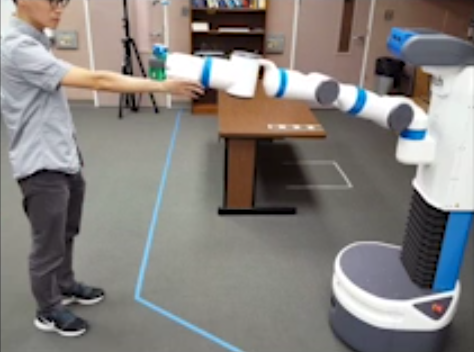}
  \end{subfigure}
  \caption{
           {\bf Responses to three different human arm movement speeds: }
           {While the robot arm is moving left to right, the human moves arm to block the robot's trajectory at different speeds.
           (a) The human arm moves slowly. The robot has enough time to predict the human arm motion, generating the smoothest and the least jerky robot trajectory.
           (b) As the human moves at a medium speed, the robot predicts the human's future motion, recognizes that it will block the robot's path, and therefore changes the trajectory upwards (at $t=0.8s$) to avoid  the obstacle and generate smooth trajectory.
           (c) When the human arm moves faster, the robot trajectory abruptly changes to move upwards (at $t=0.8s$), generating a less smooth trajectory, while avoiding the human.
           }} \label{fig:speed_response}
\end{figure*}

Figure~\ref{fig:speed_response} shows the robot's responses to three different speeds of human movements in a human-robot scenario. In this scenario, the human arm serves as a blocker or obstacle to the robot's motion from left to right.
The human moves at slow speed in (a), medium speed in (b), and fast speed in (c).
Because the human motion prediction and the robot motion planning process operate in parallel at the same time, the motion planner needs to take into account the current results of the motion predictor. If the human moves slowly, the robot motion planner is given the future predicted human motion, and therefore the planner has enough planning time to adjust the robot's trajectory. This results in smooth and collision-free trajectories of the robot.
However, if the robot moves fast, the prediction is not very accurate due to the limited processing time. At the next planning timestep, the robot's trajectory may abruptly change to avoid the current human pose, generating a jerky or non-smooth motion. This highlights how the performance of the prediction algorithm affects the smoothness of robot's trajectory.

\subsection{Benefits of our prediction algorithm}

\begin{figure}[ht]
  \centering
  \begin{subfigure}[]{0.3\linewidth}
    \centering
    \includegraphics[width=\textwidth]{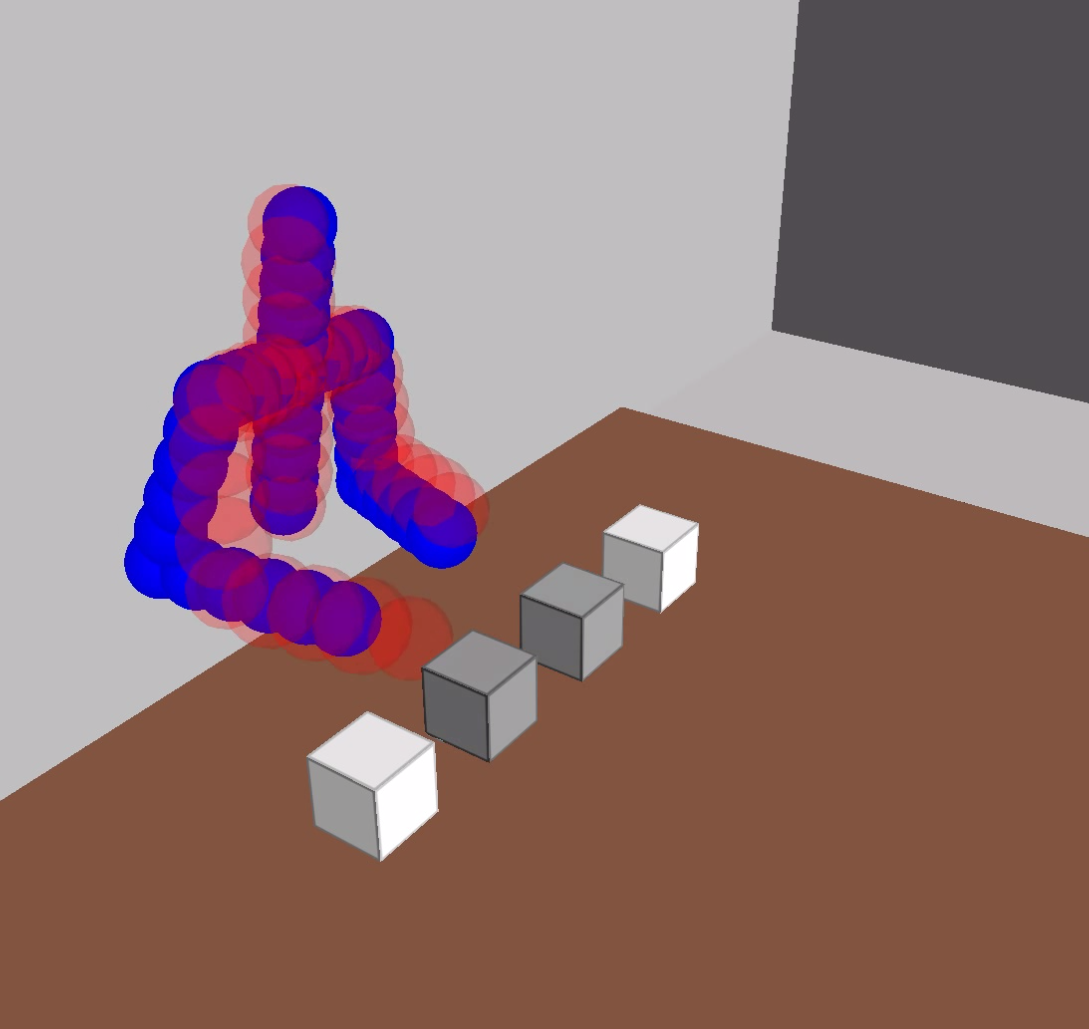}
    \caption{}
  \end{subfigure}
  \begin{subfigure}[]{0.3\linewidth}
    \centering
    \includegraphics[width=\textwidth]{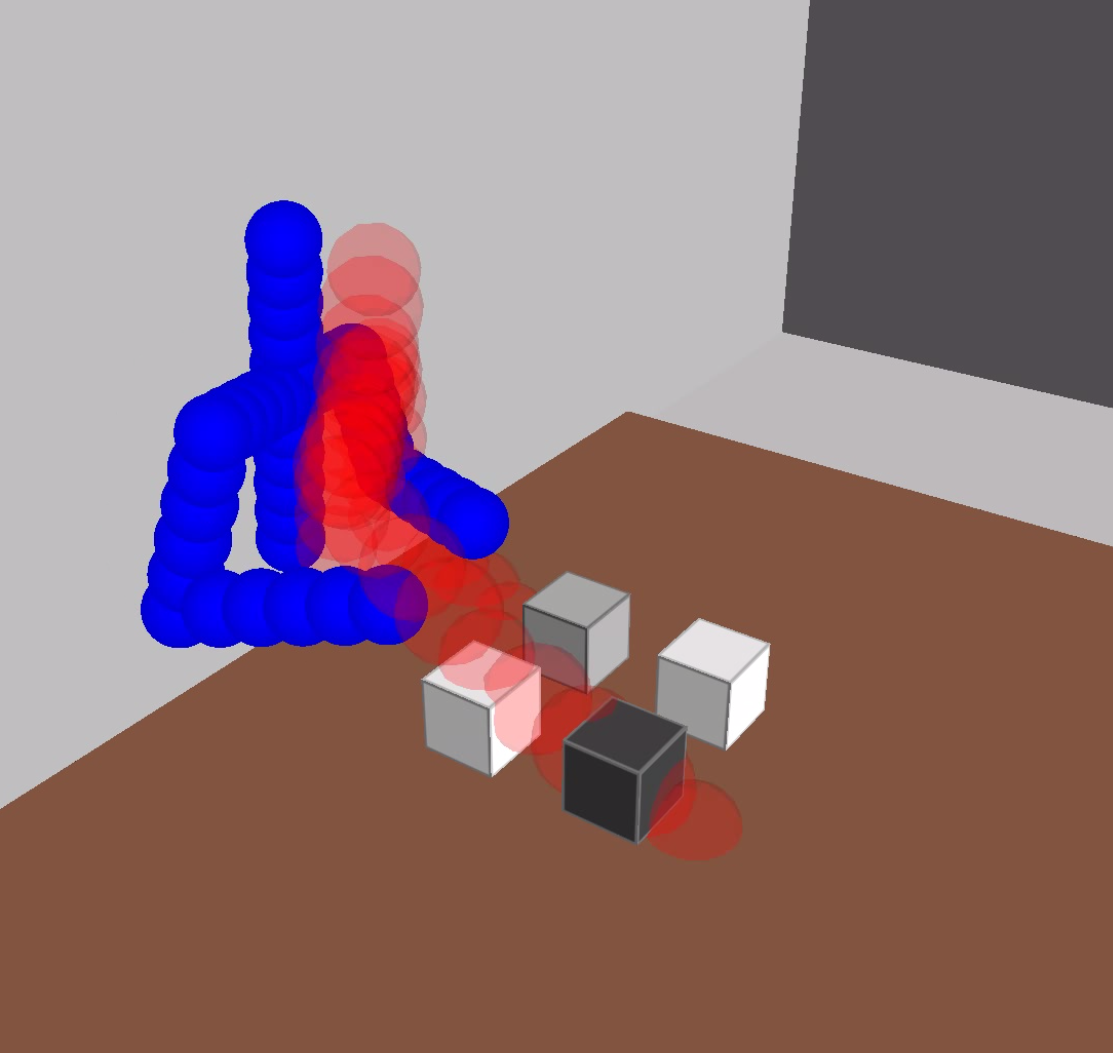}
    \caption{}
  \end{subfigure}
  \begin{subfigure}[]{0.3\linewidth}
    \centering
    \includegraphics[width=\textwidth]{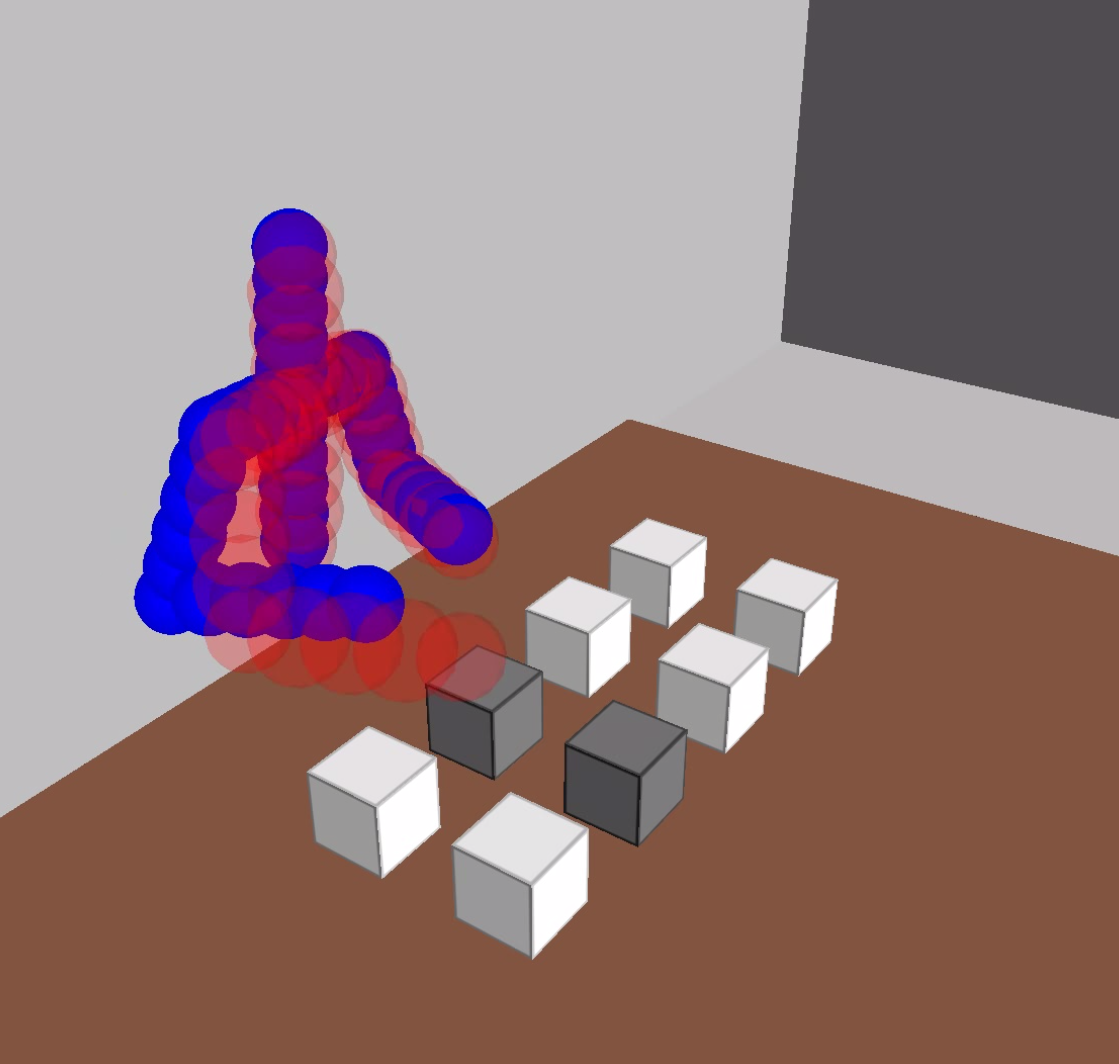}
    \caption{}
  \end{subfigure}
  \caption{{\bf Different block arrangements:} 
           {Different arrangements in terms of the positions of the blocks, results in different human motions and actions. Our planner computes their intent for safe trajectory planning. The different arrangements are: 
           (a) $1 \times 4$.
           (b) $2 \times 2$.
           (c) $2 \times 4$.
           }} \label{fig:result_arrangement}
\end{figure}

In the \textit{Different Arrangements} scenarios, the position and layout of the blocks changes.
Fig.~\ref{fig:result_arrangement} shows three different arrangements of the blocks: $1 \times 4$, $2 \times 2$ and $2 \times 4$.
In the two cases $2 \times 2$ and $2 \times 4$, where positions are arranged in two rows unlike the $1 \times 4$ scenario, the human arm blocks a movement from a front position to the back position. As a result, the robot needs to compute its trajectory accordingly.

Depending on the temporal coherence present in the human tasks, the human intention prediction may or may not improve the performance of our the task planner. It is shown in the \textit{Temporal Coherence} scenarios.
In the sequential order coherence, the human intention is predicted accurately with our approach with $100\%$ certainty.
In the random order, however, the human intention prediction step is not accurate until the human hand reaches the specific position.
The personal order varies for each human, and reduces the possibility of predicting the next human action.
When the right arm moves forward a little, $\mathit{Fetch}_0$ is predicted as the human intention with a high probability whereas $\mathit{Fetch}_1$ is predicted with low probability, even though position $1$ is closer than position $0$.

\begin{figure}[ht]
  \centering
  \begin{subfigure}[]{0.3\linewidth}
    \centering
    \includegraphics[width=\textwidth]{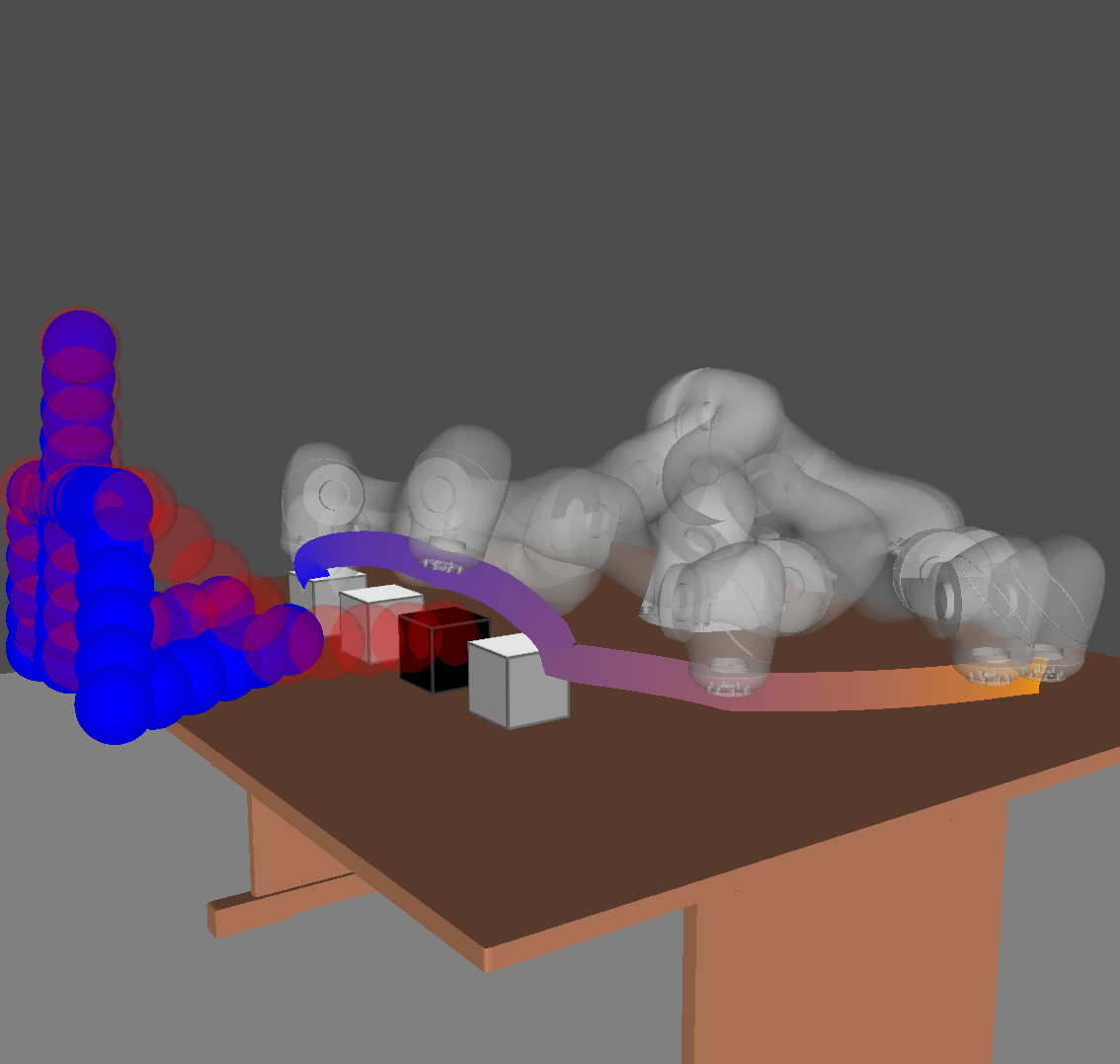}
    \caption{}
  \end{subfigure}
  \begin{subfigure}[]{0.3\linewidth}
    \centering
    \includegraphics[width=\textwidth]{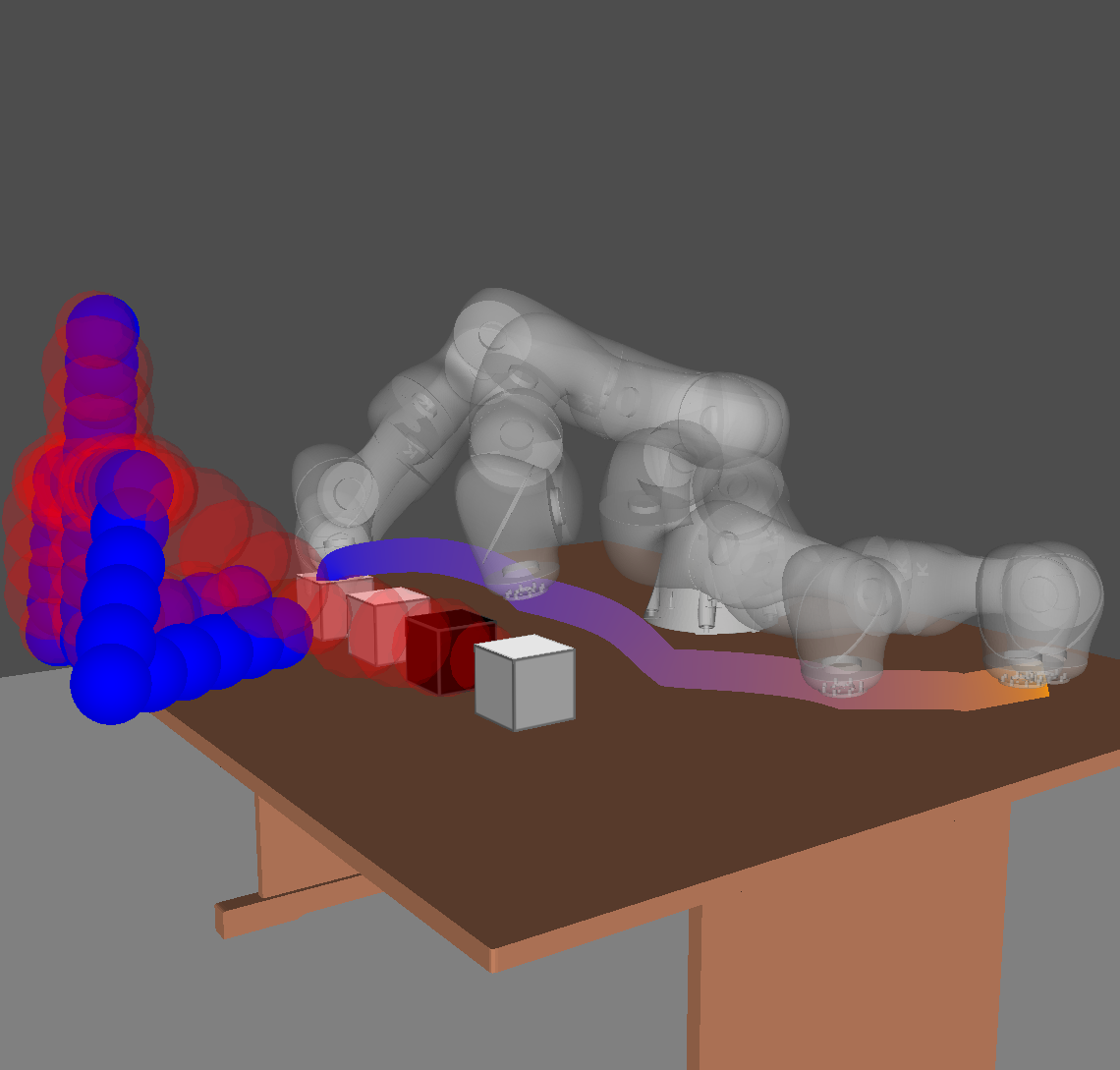}
    \caption{}
  \end{subfigure}
  \begin{subfigure}[]{0.3\linewidth}
    \centering
    \includegraphics[width=\textwidth]{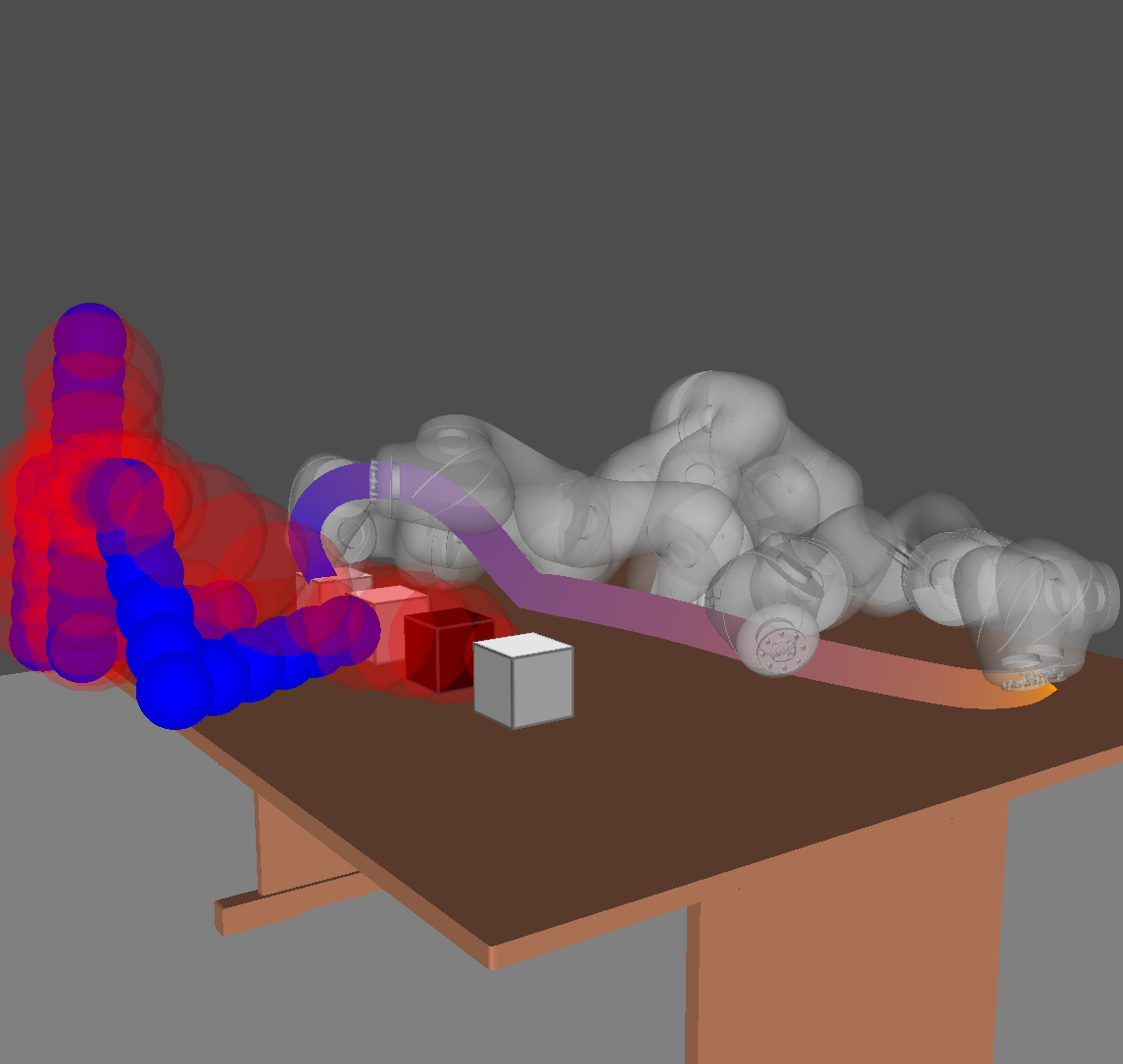}
    \caption{}
  \end{subfigure}
  \caption{{\bf Probabilistic collision checking with different confidence levels:} 
           {A collision probability less $(1 - \delta_CD)$ implies a safe trajectory. The current pose (i.e., blue spheres) and the predicted future pose (i.e. red spheres) are shown. The robot's trajectory avoids these collisions before the human performs its action. The higher the confidence level is, the longer the distance between the human arm and the robot trajectory.
           (a) $\delta_{CD} = 0.90$.
           (b) $\delta_{CD} = 0.95$.
           (c) $\delta_{CD} = 0.99$.
           }}
\end{figure}

In the \textit{Confidence Level} scenarios, we analyze the effect of confidence level $\delta_{CD}$ on the trajectory computed by the planner, the average task completion time, and the average motion planning time.
As the confidence level becomes higher, the robot may not take the smoothest and shortest path so as to compute a collision-free path that is consistent with the confidence level.

In all cases, we observe the prediction results in smoother trajectory, using the smoothness metric defined as Equation (\ref{eq:smoothness}). This is because the robot changes its path in advance before the human obstacle actually blocks the robot's shortest path if human motion prediction is used.

\subsection{Evaluation using 7-DOF Fetch robot}
We integrated our planner with 
the 7-DOF Fetch robot arm and evaluted in complex 3D workspaces.
The robot delivers four soda cans from start locations to target locations on a desk.
At the same time, the human sitting in front of the desk picks up and takes away the soda cans delivered to the target positions by the robot, which can cause collisions with the robot arm.
In order to evaluate the collision avoidance capability of our approach, the human intentionally starts moving his arm to a soda can at a target location, blocking the robot's initially planned trajectory, when the robot is delivering another can moving fast.
Our intention aware planner avoids collisions with the human arm and results in a smooth trajectory compared to motion planner results without human motion prediction.

Figure~\ref{fig:result_real_robot} shows two sequences of robot's trajectories. In the first row, the robot arm trajectory is generated an ITOMP~\cite{Park:2012:ICAPS} re-planning algorithm without human motion prediction. As the human and the robot arm move too fast to re-plan collision-free trajectory. As a result, the robot collides  (the second figure) or results in a jerky trajectory (the third figure). In the second row, our human motion prediction approach is incorporated as described in Section~\ref{sec:planning}. The robot re-plans the arm trajectory before the human actually blocks its way, resulting in collision-free path.

%% file: 8.tex
\section{Conclusions, limitations, and future work}
\label{sec:conclusions}

We present a novel intention-aware planning algorithm to compute safe robot trajectories in dynamic environments with humans performing different actions.
Our approach uses offline learning of human motions and can account for large noise in terms of depth cameras.
At runtime, our approach uses the learned human actions to predict and estimate the future motions. We use upper bounds on collision guarantees to compute safe trajectories.
We highlight the performance of our planning algorithm in complex benchmarks for human-robot cooperation in both simulated and real world scenarios with 7-DOF robots.
Compared to~\cite{park2017intention}, our improved human motion prediction model can better handle input noise and generate smoother robot trajectories.

Our approach has some limitations.
As the number of human action types increases, the number of states of MDP can increase significantly. In this case, it may be useful to use POMDP for robot motion planning under uncertainty~\cite{kurniawati2012global}.
Our probabilistic collision checking formulation is limited to environment uncertainty and does not take into account robot control errors. The performance of motion prediction algorithm depends on the variety and size of the learned data. Currently, we use supervised learning with labeled action types, but it would be useful to explore unsupervised learning based on appropriate action clustering algorithms.
Furthermore, our analysis of human motion prediction noise assumes a Gaussian process model and it would be useful to extend other noise models.
Moreover,, we would to measure the impact of robot actions on human motion, and thereby establish a two-way coupling between robot and human actions.


%% file: 9.tex
\section{Acknowledgement}
This research is supported in part by ARO grant W911NF-14-1-0437 and NSF grant 1305286.